\documentclass[letterpaper]{article}

\usepackage[margin=0.95in]{geometry}
\usepackage{amsmath}
\usepackage{amsthm}   %
\usepackage{amsfonts}       %
\usepackage{bm} %
\usepackage{nicefrac}       %
\usepackage{dsfont}

\usepackage{graphicx}
\usepackage{caption}
\usepackage{subcaption}

\usepackage{hyperref}       %
\usepackage{url}            %
\usepackage{booktabs}       %
\usepackage{tabularx}

\usepackage{multirow}
\usepackage{graphicx}
\usepackage[table,xcdraw]{xcolor}

\usepackage{nicefrac}       %
\usepackage{stmaryrd,scalerel}
 \usepackage{enumitem}
 \usepackage{authblk}
\renewcommand{\paragraph}[1]{\textbf{#1}}

\newtheorem{theorem}{Theorem}[section]
\newtheorem{proposition}[theorem]{Proposition}
\newtheorem{corollary}[theorem]{Corollary}
\theoremstyle{definition}
\newtheorem{definition}{Definition}
\newcommand{\modelname}{GSN}

\begin{document}

\title{Improving Graph Neural Network Expressivity via Subgraph Isomorphism Counting}

\author[1]{Giorgos Bouritsas\thanks{g.bouritsas@imperial.ac.uk}}
\author[1,2]{Fabrizio Frasca\thanks{ffrasca@twitter.com}}
\author[1]{Stefanos Zafeiriou \thanks{s.zafeiriou@imperial.ac.uk}}
\author[1,2]{Michael M. Bronstein\thanks{mbronstein@twitter.com}}
\affil[1]{Imperial College London, UK}
\affil[2]{Twitter, UK}
\date{}

\maketitle

\begin{abstract}
While Graph Neural Networks (GNNs) have achieved remarkable results in a variety of applications, recent studies exposed important shortcomings in their ability to capture the structure of the underlying graph. It has been shown that the expressive power of standard GNNs is bounded by the Weisfeiler-Leman (WL) graph isomorphism test, from which they inherit proven limitations such as the inability to detect and count graph substructures. On the other hand, there is significant empirical evidence, e.g. in network science and bioinformatics, that substructures are often intimately related to downstream tasks.
To this end, we propose ``Graph Substructure Networks'' (GSN), a topologically-aware message passing scheme based on substructure encoding. We theoretically analyse the expressive power of our architecture, showing that it is strictly more expressive than the WL test, and provide sufficient conditions for universality. Importantly,  we do not attempt to adhere to the WL hierarchy; this allows us to retain multiple attractive properties of standard GNNs such as locality and linear network complexity, while being able to disambiguate even hard instances of graph isomorphism. We perform an extensive experimental evaluation on graph classification and regression tasks and obtain state-of-the-art results in diverse real-world settings including molecular graphs and social networks. The code is publicly available at \url{https://github.com/gbouritsas/graph-substructure-networks}.
\end{abstract}

\section{Introduction}\label{sec:introduction}

The field of graph representation learning has undergone a rapid growth in the past few years. 
In particular, Graph Neural Networks (GNNs), a family of neural architectures designed for irregularly structured data, have been successfully applied to problems 
ranging from social networks and recommender systems \cite{ying2018graph} to bioinformatics \cite{fout2017protein, gainza2020deciphering}, chemistry \cite{duvenaud2015convolutional, gilmer2017neural, sanchez2019machine} and physics \cite{DBLP:conf/icml/KipfFWWZ18, battaglia2016interaction}, to name a few. 
Most GNN architectures are based on message passing \cite{gilmer2017neural}, where the representation of each node is iteratively updated by aggregating information from its neighbours. 

A crucial difference from traditional neural networks operating on grid-structured data is the absence of canonical ordering of the nodes in a graph. To address this, the aggregation function is constructed to be invariant to neighbourhood permutations and, as a consequence, to graph isomorphism. This kind of symmetry is not always desirable and thus different inductive biases that disambiguate the neighbours have been proposed. For instance, in geometric graphs, such as 3D molecular graphs and meshes, directional biases are usually employed in order to model the positional information of the nodes  \cite{masci2015geodesic,monti2017geometric,bouritsas2019neural,klicpera_dimenet_2020, de2020gauge}; for proteins, ordering information is used to disambiguate amino-acids at different positions in the sequence \cite{ingraham2019generative}; in multi-relational knowledge graphs, a different aggregation is performed for each relation type \cite{schlichtkrull2018modeling}. 

The structure of the graph itself does not usually explicitly take part in the aggregation function. In fact, most models rely on multiple message passing steps as a means for each node to discover the global structure of the graph. However, since message-passing GNNs are at most as powerful as the Weisfeiler Leman test (WL) \cite{xu2018how, morris2019weisfeiler}, they are limited in their abilities to adequately exploit the graph structure, e.g. by  counting  substructures \cite{DBLP:conf/fct/ArvindFKV19, chen2020can}. This uncovers a crucial limitation of GNNs, as substructures have been widely recognised as important in the study of complex networks. For example, in molecular chemistry, functional groups and rings are related to a plethora of chemical properties, while cliques are related to protein complexes in Protein-Protein Interaction networks and community structure in social networks, respectively \cite{Granovetter82thestrength, girvan2002community}. 

Therefore, three major questions arise when designing GNN architectures: (a) How to go beyond \textit{isotropic}, i.e., locally symmetric, aggregation functions ? (b) How to ensure that GNNs are aware of the \textit{structural chatacteristics} of the graph? (c) How to achieve the above two without sacrificing \textit{invariance to isomorphism} and hence the ability of GNNs to generalise?

In this work we attempt to simultaneously provide an answer to the above. We propose to break local symmetries by introducing structural information in the aggregation function, hence addressing (a) and (b). In particular, the contribution of each neighbour (\textit{message}) is transformed differently depending on its structural relationship with the central node. This relationship is expressed by counting the appearance of certain substructures. Since substructure counts are \textit{vertex invariants}, i.e. they are invariant to vertex permutations, it is easy to see that the resulting GNN will be invariant to isomorphism, hence also addressing (c). Moreover, by choosing the substructures, one can provide the model with different inductive biases, based on the graph distribution at hand.

We characterise the expressivity of our message-passing scheme, coined as \textit{Graph Substructure Network (GSN)}, 
showing that GSN is strictly more expressive than traditional GNNs for the vast majority of substructures, while retaining the locality of message passing, as opposed to higher-order methods \cite{DBLP:conf/iclr/MaronBSL19,pmlr-v97-maron19a, maron2019provably, morris2019weisfeiler} that follow the WL hierarchy (see Section \ref{sec:prelims}). 
In the limit, our model can yield a unique representation for every isomorphism class and is thus universal. 
We provide an extensive experimental evaluation on hard instances of graph isomorphism testing (strongly regular graphs), as well as on real-world networks from the social and biological domains, including the recently introduced large-scale benchmarks \cite{dwivedi2020benchmarking, DBLP:ogb}. We observe that when choosing the structural inductive biases based on domain-specific knowledge, \modelname\ achieves state-of-the-art results.
\section{Preliminaries}\label{sec:prelims}

Let $G=(\mathcal{V}_G, \mathcal{E}_G)$ be a graph with vertex set $\mathcal{V}_G$ and edge set $\mathcal{E}_G$, directed or undirected. A subgraph ${G_S = (\mathcal{V}_{G_S}, \mathcal{E}_{G_S})}$ of $G$ is any graph with $\mathcal{V}_{G_S}\subseteq\mathcal{V}_G, \ \mathcal{E}_{G_S}\subseteq\mathcal{E}_G$. When $\mathcal{E}_{G_S}$ includes all the edges of $G$ with endpoints in $\mathcal{V}_{G_S}$, i.e.,  ${\mathcal{E}_{G_S} = \mathcal{E}_G \cap \big( \mathcal{V}_{G_S} \times \mathcal{V}_{G_S}\big)}$,
the subgraph is said to be \textit{induced}. 

\subsection{Isomorphism \& Automoprhism}
Two graphs $G, H$ are {\em isomorphic} (denoted $H\simeq G$), if there exists an adjacency-preserving bijective mapping ({\em isomorphism}) $f: \mathcal{V}_G  \rightarrow \mathcal{V}_H$, i.e., $(v,u) \in \mathcal{E}_G$ iff $(f(v),f(u)) \in \mathcal{E}_H$. Given some small graph $H$, the \textit{subgraph isomorphism} problem amounts to finding a subgraph $G_S$ of $G$ such that $G_S\simeq H$. An \textit{automorphism} of $H$ is an isomorphism that maps $H$ onto itself. The set of all the unique automorphisms forms the \textit{automorphism group} of the graph, denoted as $\mathrm{Aut}(H)$, which contains all the possible symmetries of the graph. 

The automorphism group yields a partition of the vertices into disjoint subsets of $\mathcal{V}_H$ called \textit{orbits}. Intuitively, this concept allows us to group the vertices based on their \textit{structural roles}, e.g. the endpoint vertices of a path, or all the vertices of a cycle (see Figure \ref{fig:method_illustration}). Formally, the
orbit of a vertex $v \in \mathcal{V}_H$ is the set of vertices to which it can be mapped via an automorphism: ${\mathrm{Orb}(v) = \{u \in  \mathcal{V}_H \ |  \ \exists g \in \mathrm{Aut}(H) \,\, \mathrm{s.t.} \,\, g(u) = v \}}$, and the set of all orbits ${H \setminus  \mathrm{Aut}(H) = \{\mathrm{Orb}(v) \ | \ v \in \mathcal{V}_H\}}$ is usually called the \textit{quotient} of the automorphism when it acts on the graph $H$. We are interested in the unique elements of this set that we will denote as $\{O^V_{H,1}, O^V_{H,2}, \dots ,O^V_{H, d_H}\}$, where $d_H$ is the cardinality of the quotient.

Analogously, we define edge structural roles via \textit{edge automorphisms}, i.e., bijective mappings from the edge set onto itself, that preserve edge adjacency (two edges are adjacent if they share a common endpoint). In particular, every vertex automorphism $g$ induces an edge automorphism by mapping each edge $(u,v)$ to $(g(u), g(v))$. \footnote{Note that the edge automorphism group is larger than that of induced automorphisms, but strictly larger only for 3 trivial cases \cite{Whitney1932}. However, induced automorphisms provide a more natural way to express edge structural roles.} In the same way as before, we construct the edge automorphism group, from which we deduce the partition of the edge set in {\em edge orbits} $\{O_{H,1}^E, O_{H,2}^E, \dots ,O_{H, d_H}^E\}$.

\subsection{Weisfeiler-Leman tests}
The {\em Weisfeiler-Leman graph-isomorphism test} \cite{weisfeilerreduction}, also known as naive vertex refinement, \textit{1-WL}, or just \textit{WL}), is a fast heuristic to decide if two graphs are isomorphic or not. The WL test proceeds as follows: every vertex $v$ is initially assigned a colour $c^0_v$ that is later iteratively refined by aggregating neighbouring information: 
\begin{equation}
c^{t+1}_v = \mathrm{HASH}\Big(c^{t}_v, \  \Lbag c^{t}_u\Rbag_{u\in\mathcal{N}_v}\Big),
\end{equation}
where $\Lbag\cdot\Rbag$ denotes a multiset (a set that allows element repetitions) and $\mathcal{N}(v)$ is the neighbourhood of $v$. The WL algorithm terminates when the colours stop changing, and outputs a histogram of colours. Two graphs with different histograms are non-isomorphic; if the histograms are identical, the graphs are possibly, but not necessarily, isomorphic. Note that the neighbour aggregation in the WL test is a form of message passing, and GNNs are the learnable analogue. 

A series of works on improving GNN expressivity mimic the higher-order generalisations of WL, known as $k$-WL and $k$-Folklore WL (WL hierarchy) and operate on $k$-tuples of nodes (see Appendix B.1). The $(k+1)$-FWL is strictly stronger than $k$-FWL, $k$-FWL is as strong as $(k+1)$-WL and 2-FWL is strictly stronger than the simple 1-WL test. 
\section{Graph Substructure Networks}

Graphs consist of nodes (or edges) with repeated structural roles. Thus, it is natural for a neural network to treat them in a similar manner, akin to weight sharing between local patches in CNNs for images \cite{lecun1989backpropagation} or positional encodings in language models for sequential data \cite{DBLP:conf/nips/SukhbaatarSWF15, DBLP:conf/icml/GehringAGYD17, vaswani2017attention}. 
Nevertheless, GNNs are usually unaware of the nodes' different structural roles, since all nodes are treated equally when performing local operations. Despite the initial intuition that the neural network would be able to discover these roles by constructing deeper architectures, it has been shown that GNNs are ill-suited for this purpose and are blind to the existence of structural properties, e.g. triangles or larger cycles \cite{chen2020can, DBLP:conf/fct/ArvindFKV19}. 

To this end, we propose to explicitly encode structural roles as part of message passing, in order to capture richer topological properties. Our method draws inspiration from \cite{Loukas2020What}, where it was shown that GNNs become universal when the nodes in the graph are uniquely identified, i.e when they are equipped with different features. However, it is not clear how to choose these identifiers in a permutation equivariant way.
Structural roles, when treated as identifiers, although not necessarily unique, are not only permutation equivariant, but also more amenable to generalisation due to their repetition across different graphs. Thus, they can constitute a trade-off between uniqueness and generalisation.

    \begin{figure}[!t]
      \centering
      \includegraphics[width=0.7\textwidth, scale=1]{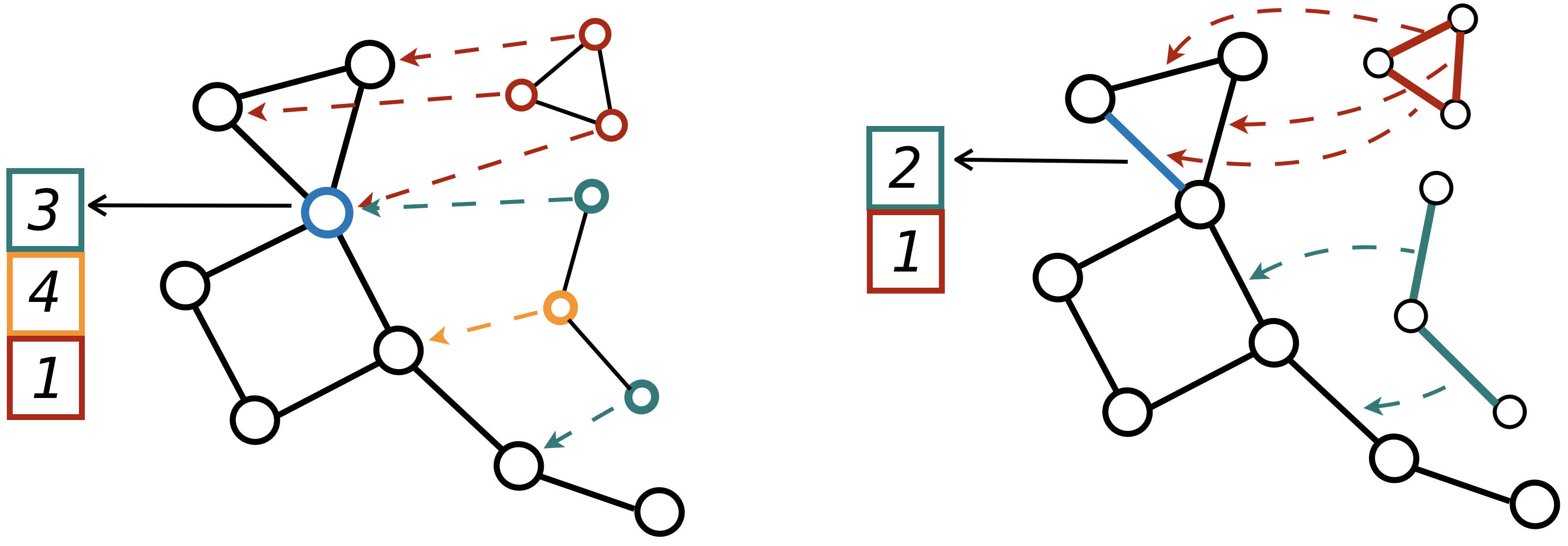}
      \caption{\emph{Node} (left) and \emph{edge} (right) induced subgraph counting for a 3-cycle and a 3-path. Counts are reported for the blue node on the left and for the blue edge on the right. Different colors depict orbits.}
      \label{fig:method_illustration}
    \end{figure}

\subsection{Structural features} Structural roles are encoded into features by counting the appearance of certain substructures. 
Define a set of small (connected) graphs ${\mathcal{H} = \{H_1,H_2\dots H_K\}}$, for example cycles of fixed length or cliques.  For each graph $H \in \mathcal{H}$, we first find its isomorphic subgraphs in $G$ denoted as $G_S$. For each node $v \in \mathcal{V}_{G_S}$ we infer its role w.r.t. $H$ by obtaining the orbit of its mapping $f(v)$ in $H$, $\mathrm{Orb}_H(f(v))$. By counting all the possible appearances of different orbits in $v$, we obtain the {\em vertex structural feature} $\mathbf{x}_{H}^V (v)$ of $v$, defined as follows. For all $i \in \{1,\dots, d_H\}$:

\begin{equation}
       x_{H,i}^V (v) = \bigg| \Big\{G_S \simeq H  \ \big| \ v \in \mathcal{V}_{G_S}, \ f(v) \in O_{H,i}^V \}\bigg|.
\end{equation}

Note that there exist $|\mathrm{Aut}(H)|$ different functions $f$ that can map a subgraph $G_S$ to $H$, but any of those can be used to determine the orbit mapping of each node $v$.
By combining the counts from different substructures in $\mathcal{H}$ and different orbits, we obtain the feature vector 
$\mathbf{x}^V_v = [\mathbf{x}_{H_1}^V(v), \hdots, \mathbf{x}_{H_K}^V(v)] \in \mathbb{N}^{D\times 1}$ of dimension ${D = {\sum_{H_i\in\mathcal{H}}d_{H_i}}}$. 

Similarly, we can define {\em edge structural features} $\mathbf{x}_{H,i}^E (u,v)$ by counting occurrences of edge automorphism orbits:

\begin{equation}
    x_{H,i}^E (u,v) = \bigg| \Big\{G_S \simeq H \ \big| \ (u,v) \in \mathcal{E}_{G_S}, \ (f(u), f(v)) \in O_{H,i}^E \Big\} \bigg|,
\end{equation}
and the combined edge features ${\mathbf{x}^E_{u,v} = [\mathbf{x}_{H_1}^E(u,v), \hdots, \mathbf{x}_{H_K}^E(u,v)]}$. An example of vertex and edge structural features is illustrated in Figure \ref{fig:method_illustration}.

\subsection{Structure-aware message passing} The key building block of our architecture is the graph substructure layer, defined in a general manner as a Message Passing Neural Network (MPNN) \cite{gilmer2017neural}, where now the messages from the neighbouring nodes also contain the structural information. In particular, each node $v$ updates its state $\mathbf{h}^t_v$ by combining its previous state with the aggregated messages:
\begin{align}\label{eq:gsn}
    \mathbf{h}^{t+1}_v &=  \mathrm{UP}^{t+1}\big(\mathbf{h}^{t}_v, \mathbf{m}^{t+1}_v\big)\\
    \mathbf{m}^{t+1}_v &=\left\{
        \begin{array}{l}
         M^{t+1}\bigg(\Lbag(\mathbf{h}^{t}_v, \mathbf{h}^{t}_u, \mathbf{x}^V_v, \mathbf{x}^V_u, \mathbf{e}_{u,v})\Rbag_{u\in \mathcal{N}(v)}\bigg) \ (\textbf{GSN-v})\\
         \quad \quad \quad \quad \quad \quad \text{ or }\\
         M^{t+1}\bigg(\Lbag(\mathbf{h}^{t}_v, \mathbf{h}^{t}_u, \mathbf{x}^E_{u,v}, \mathbf{e}_{u,v})\Rbag_{u\in \mathcal{N}(v)} \bigg) \ (\textbf{GSN-e}),
                 \end{array} 
                 \right.\label{eq:msg_fn}
\end{align}
where $\mathrm{UP}^{t+1}$ is an arbitrary function approximator (e.g. a MLP), $M^{t+1}$ is the neighborhood aggregation function, i.e. an arbitrary function on multisets (e.g., of the form $\sum_{u\in \mathcal{N}(v)} \text{MLP}(\cdot)$) and $ \mathbf{e}_{u,v}$ are the edge features. The two variants, named \textit{GSN-v} and \textit{GSN-e}, correspond to vertex- or edge-counts, respectively, which are analogous to absolute and relative positional encodings in language models \cite{DBLP:conf/naacl/ShawUV18, DBLP:conf/acl/DaiYYCLS19}.

It is important to note here that contrary to identifier-based GNNs \cite{Loukas2020What, sato2020random, clip_ijcai20} that obtain universality at the expense of permutation equivariance (since the identifiers are arbitrarily chosen with the sole requirement of being unique), GSNs retain this property, hence they are by construction ivariant to isomorphism. This stems from the fact that the process generating our structural identifiers (i.e. subgraph isomorphism) is permutation equivariant itself (proof provided in the Appendix A.1).

\begin{figure*}[t]
\centering
\begin{subfigure}{0.45\textwidth}
  \centering
  \includegraphics[width=0.8\textwidth]{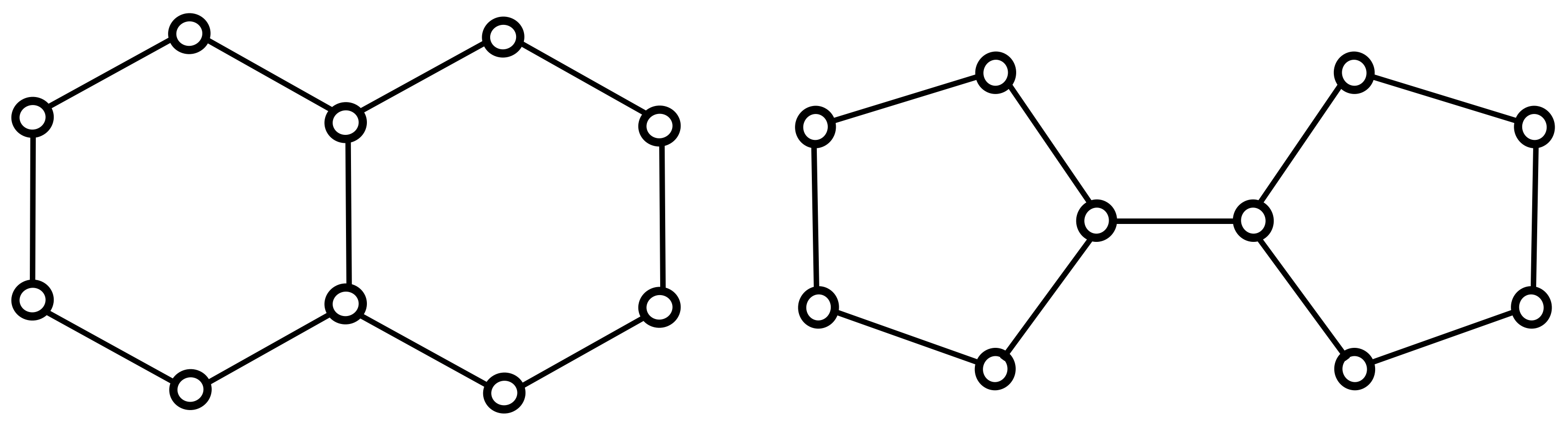}
  \label{fig:molecules}
\end{subfigure}
\hspace{1cm}
\begin{subfigure}{0.45\textwidth}
  \centering
     \includegraphics[width=0.8\textwidth]{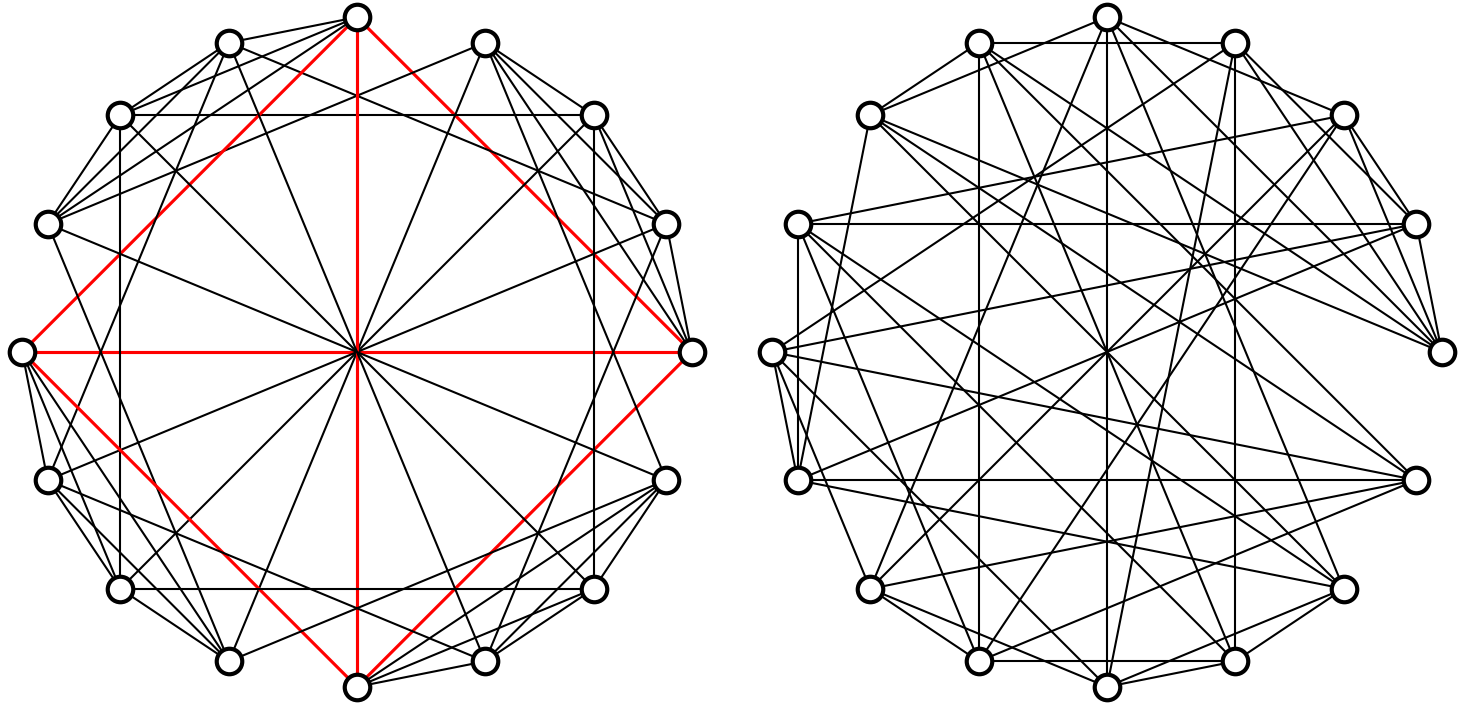}
\end{subfigure}
    \captionsetup[figure]{skip=\abovecaptionskip}
  \caption{(Left) \emph{Decalin} and \emph{Bicyclopentyl}: Non-isomorphic molecular graphs than can be distinguished by \modelname, but not the by the WL test \cite{sato2020survey} (nodes represent carbon atoms and edges chemical bonds).
  (Right) \emph{Rook's 4x4 graph} and the \emph{Shrikhande graph}: the smallest pair of strongly regular non-isomorphic graphs with the same parameters SR(16,6,2,2). \modelname\ can distinguish them with 4-clique counts, while 2-FWL fails.
}
 \label{fig:graph_examples}
\end{figure*}

\subsection{How powerful are \modelname s?}

We now turn to the expressive power of \modelname s in comparison to MPNNs and the WL tests, a key tool for the theoretical analysis of the expressivity of graph neural networks so far. Since GSN is a generalisation of MPNNs, it is easy to see that it is at least as powerful. Importantly, \modelname s have the capacity to learn functions that traditional MPNNs cannot learn. The following observation derives directly from the analysis of the counting abilities of the 1-WL test  \cite{DBLP:conf/fct/ArvindFKV19} and its extension to MPNNs \cite{chen2020can} (for proofs, see Appendices A.2-A.4). 

\begin{theorem}\label{thrm:GSN vs MPNN}
\modelname\ is strictly more powerful than MPNN and the 1-WL test when one of the following holds:
\begin{itemize}
    \item $H$ is any graph except for star graphs of any size, and structural features are inferred by subgraph matching, i.e. we count all subgraphs $G_S \cong H$ for which it holds that $\mathcal{E}_{G_S} \subseteq \mathcal{E}_G$. Or,
    \item $H$ is any graph except for single edges and single nodes,     and structural features are inferred by \textbf{induced} subgraph matching, i.e. we count all subgraphs $G_S \cong H$ for which it holds that $\mathcal{E}_{G_S} = \mathcal{E}_G \cap \big( \mathcal{V}_{G_S} \times \mathcal{V}_{G_S}\big)$.
\end{itemize}
\end{theorem}
\begin{proof}
It is easy to see that \modelname\ model class contains MPNNs, and is thus at least as expressive.  We can also show that \modelname\ is at least as expressive as the 1-WL test by repurposing the proof of Theorem 3 in \cite{xu2018how} (see Appendix A.2)

 Given the first part of the proposition, in order to show that \modelname s are strictly more expressive than the 1-WL test, it suffices to show that \modelname\ can distinguish a pair of graphs that 1-WL deems isomorphic. \cite{DBLP:conf/fct/ArvindFKV19} showed that 1-WL, and consequently MPNNs, can count only \textit{forests of stars}. Thus, if the subgraphs are required to be connected, then they can only be star graphs of any size (note that this contains single nodes and single edges). In addition, \cite{chen2020can} showed that 1-WL, and consequently MPNNs, cannot count any connected \textbf{induced} subgraph with 3 or more nodes, i.e. any connected subgraph apart from single nodes and single edges. 
 
 If $H$ is a substructure that 1-WL cannot learn to count, i.e. the ones mentioned above, then there is at least one pair of graphs with different number of counts of $H$, that 1-WL deems isomorphic. Thus, by assigning counting features to the nodes/edges of the two graphs based on appearances of $H$, a \modelname \ can obtain different representations for $G_1$ and $G_2$ by summing up the features. Hence, $G_1$, $G_2$ are deemed non-isomorphic. An example is depicted in Figure \ref{fig:graph_examples} (left), where the two  non-isomorphic graphs are distinguishable by GSN via e.g. cycle counting, but not by 1-WL.
\end{proof}

\noindent\textbf{Universality. } A natural question that emerges is what are the sufficient conditions under which GSN can solve graph isomorphism. This would entail that GSN is a universal approximator of functions defined on graphs \cite{clip_ijcai20, chen2019equivalence}. To address this, we can examine whether there exists a specific substructure collection that can completely characterise each graph. As of today, we are not aware of any results in graph theory that can guarantee the reconstruction of a graph from a smaller collection of its subgraphs. However, the \textit{Reconstruction Conjecture} \cite{kelly1957congruence, ulam1960collection}, states that a graph with size $n\geq3$ can be reconstructed from its vertex-deleted subgraphs (proven for $n\leq 11$ \cite{mckay1997small}). Consequently, (proof in the Appendix A.3):
\begin{corollary}\label{thrm:recon_conj}
If the Reconstruction Conjecture holds and the substructure collection $\mathcal{H}$ contains all graphs of size $k= n-1$, then GSN can distinguish all non-isomorphic graphs of size $n$ and is therefore universal. 
\end{corollary}

\noindent\textbf{GSN-v vs GSN-e. } We can also examine the expressive power of the two proposed variants. A crucial observation that we make is that for each graph $H$ in the collection, the vertex structural identifiers can be reconstructed by the corresponding edge identifiers. Thus, we can show that for every GSN-v there exists a GSN-e that can simulate the behaviour of the former (proof provided in the Appendix A.4).

\begin{theorem}\label{thrm: GSN-v vs GSN-e}
For a given subgraph collection $\mathcal{H}$, let $C^V$ the set of functions that can be expressed by a GSN-v with arbitrary depth and with, and $C^E$ the set of functions that can be expressed by a GSN-e with the same properties. Then, it holds that $C^E \supseteq C^V$, or in other words GSN-e is at least as expressive as GSN-v.
\end{theorem}

\noindent\textbf{Comparison with higher-order WL tests. } Finally, the expressive power of GSN can be compared to higher-order versions of the WL test. In particular, for each $k$-th order Folklore WL test in the hierarchy, it is known that there exists a family of graphs that will make the test fail. These are known in the literature as $k$-isoregular graphs \cite{douglas2011weisfeiler}, and the most well-known example is the \textit{Strongly Regular} (SR) graph family, for $k=2$ (more details can be found in Appendix B). 

Hence, if we can find a substructure collection that allows GSN to distinguish certain pairs from these families, then this guarantees that the corresponding $k$-FWL test is no stronger than GSN. In this work, we identify such counterexamples for the 2-FWL test. Formally:

\begin{proposition}\label{thrm:GSN vs 2-FWL}
There exist substructure families with $\mathcal{O}(1)$ size, i.e., independent of the size of the graph $n$, such that 2-FWL is no stronger than GSN.
\end{proposition}

We provide numerous counterexamples that prove this claim. Figure  \ref{fig:graph_examples} (right) provides a typical pair of SR graphs that can be distinguished with a 4-clique, while in section \ref{sec:synthetic} this is extended to a large-scale study, where other constant size substructures (paths, cycles and cliques) can achieve similar results. 

Although it is not clear if there exists a certain substructure collection that results in GSNs that align with the WL hierarchy, we stress that this is not a necessary condition in order to design more powerful GNNs. In particular, despite the increase in expressivity, k-WL tests are not only more computationally involved, but they also process the graph in a non-local fashion. However, locality is presumed to be a strong inductive bias of GNNs and key to their excellent performance in real-world scenarios.

\begin{figure*}[t!]
     \centering
     \begin{subfigure}[t]{0.33\textwidth}
         \centering
         \includegraphics[width=\textwidth]{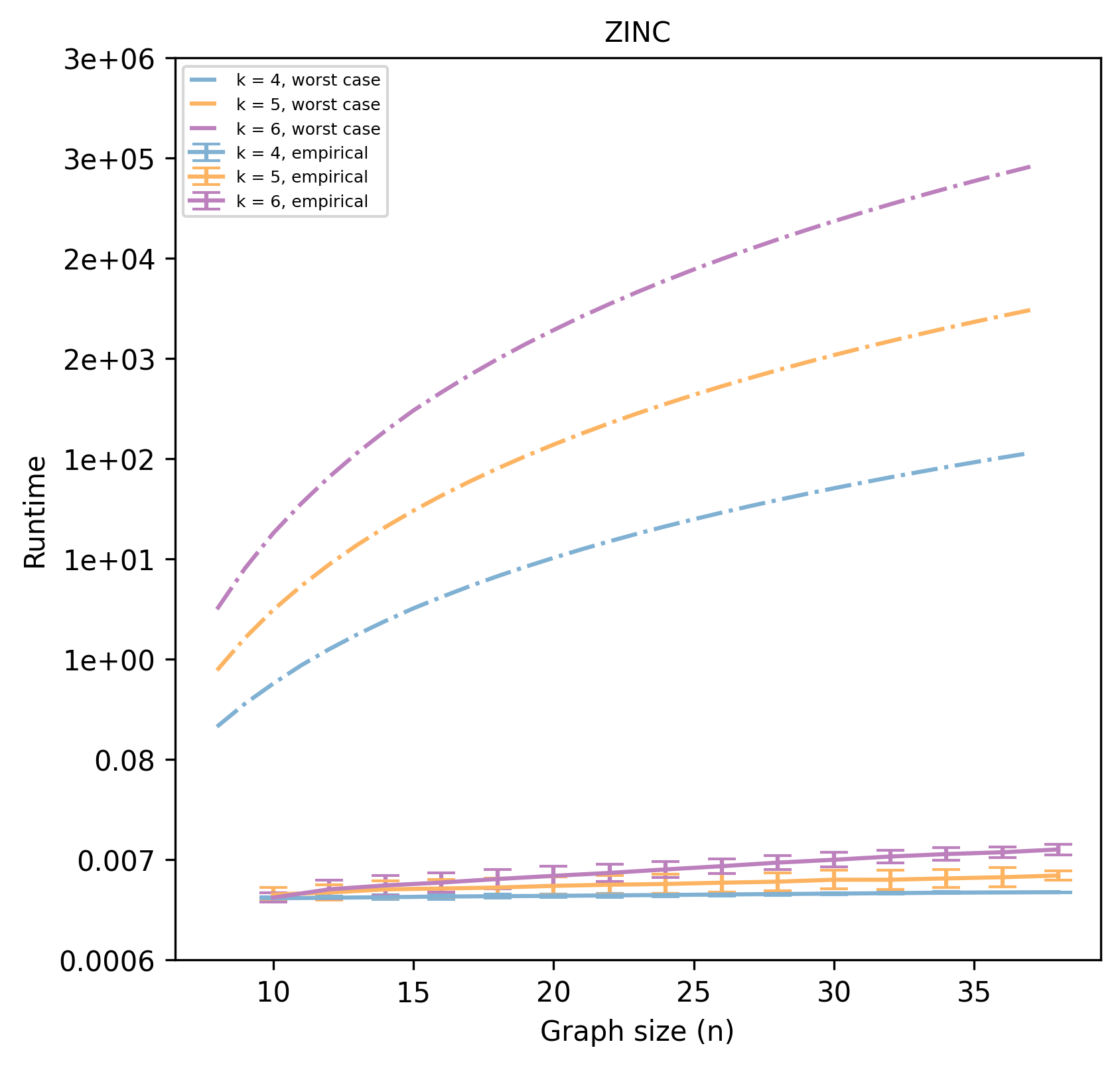}
         \label{fig:y equals x}
     \end{subfigure}
     \begin{subfigure}[t]{0.32\textwidth}
         \centering
         \includegraphics[width=\textwidth]{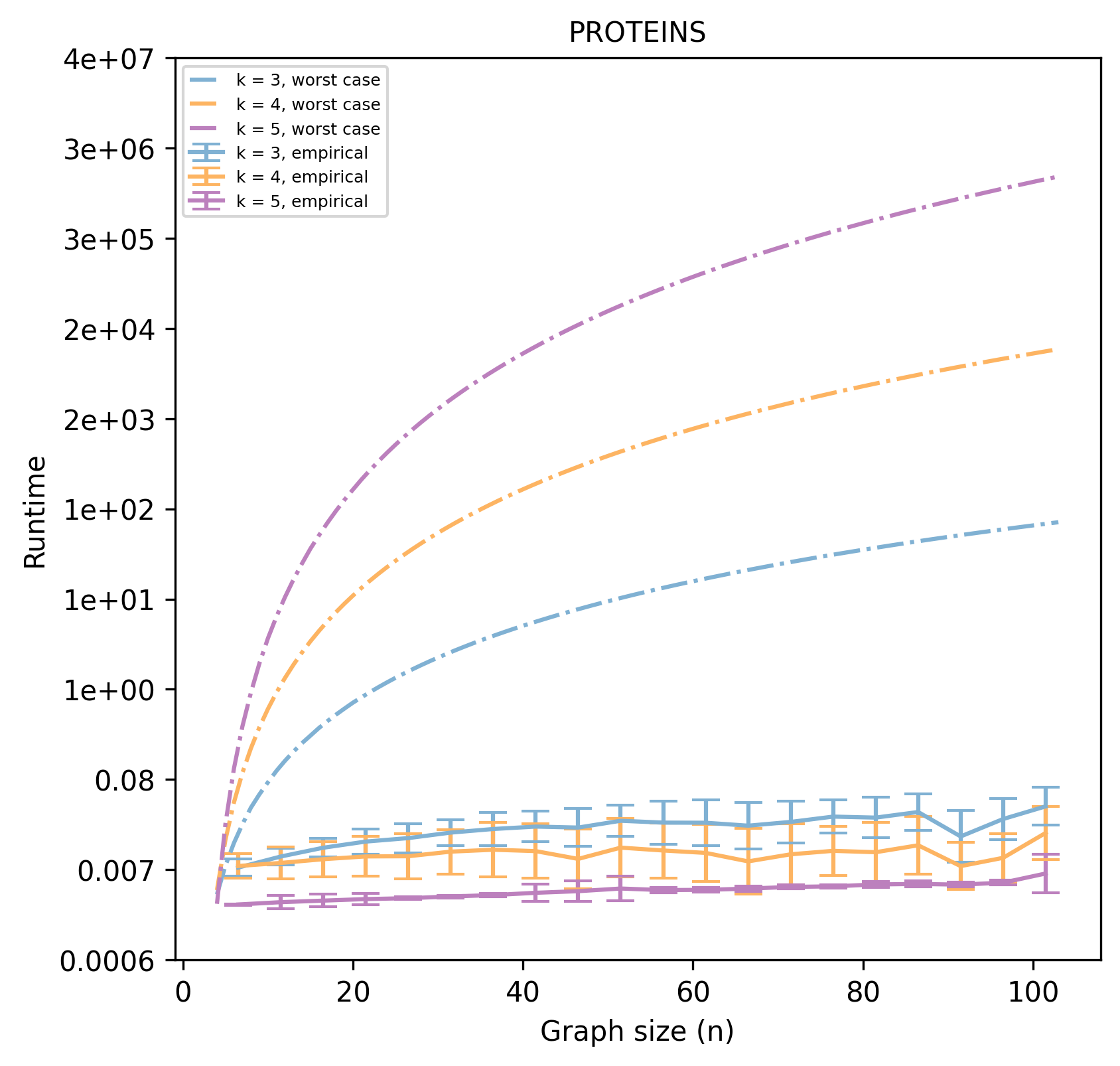}
         \label{fig:three sin x}
     \end{subfigure}
     \begin{subfigure}[t]{0.33\textwidth}
         \centering
         \includegraphics[width=\textwidth]{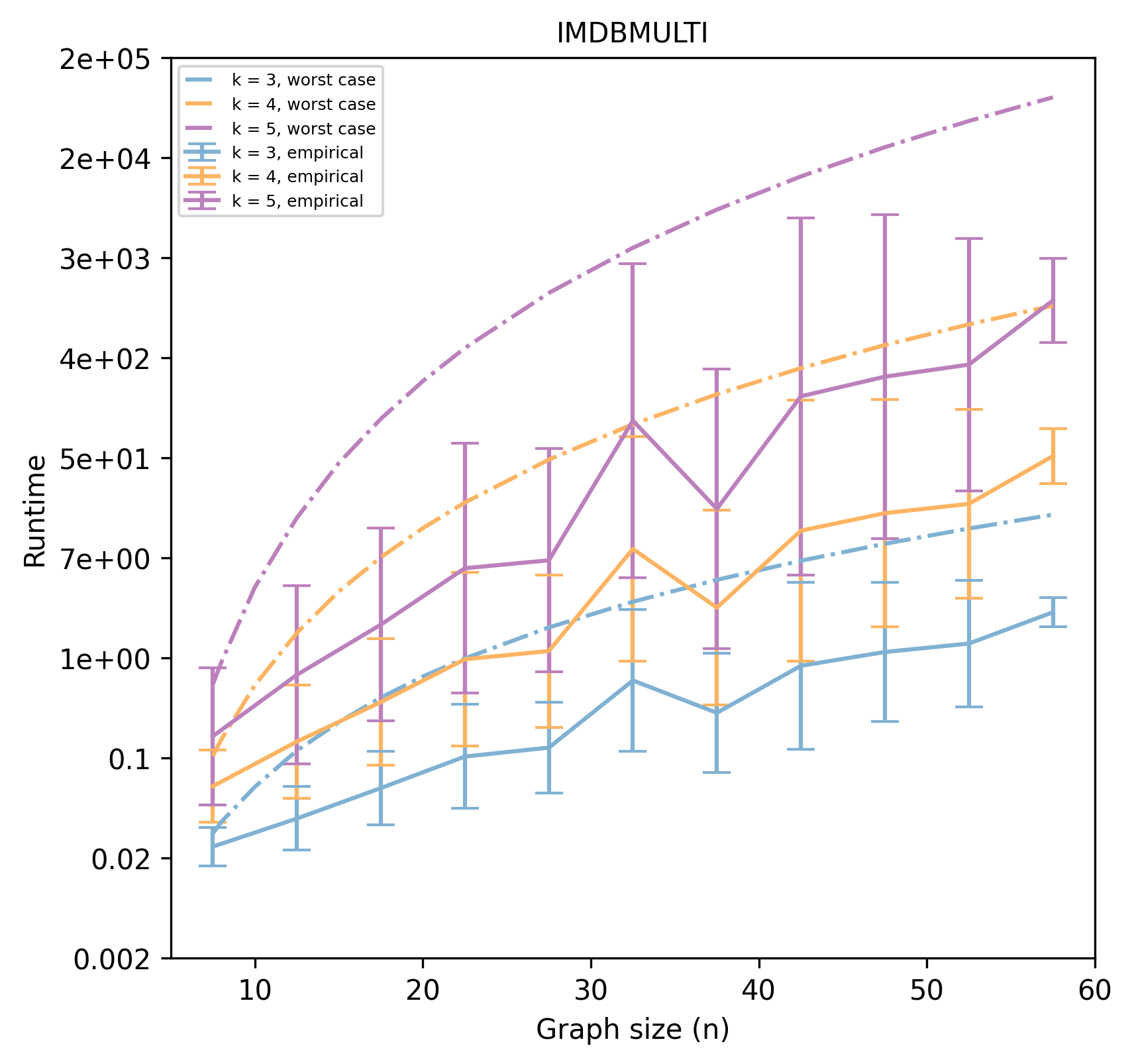}
         \label{fig:five over x}
     \end{subfigure}
        \caption{Empirical (solid) vs worst case (dashed) runtime for different graph distributions (in seconds, log scale). For each distribution we count the best performing (and frequent) substructures of increasing sizes $k$. 
        Computational complexity for real-life graphs is significantly better than the worst case. 
        }
        \label{fig:complexity}
\end{figure*}

\subsection {How to choose the substructures?}
\textbf{Expressivity. } The Reconstruction Conjecture provides a sufficient, albeit impractical condition for universality. This motivates us to analyse the constant size case $k=\mathcal{O}(1)$ for practical scenarios, similar to the argument put forward for hard instances of graph isomorphism (Proposition \ref{thrm:GSN vs 2-FWL}).

In particular, one can count only the most discriminative subgraphs, i.e. the ones that can achieve the maximum possible vertex disambiguation, similarly to identifier-based approaches. Whenever these subgraph counts can provide a unique identification of the vertices, then universality will also hold (Corollary 3.1. in \cite{Loukas2020What}).

We conjecture, that in real-world scenarios the number of subgraphs needed for unique, or near-unique identification, are far fewer than those dictated by Corollary \ref{thrm:recon_conj}. This is consistent with our experimental findings, where we observed that certain small substructures such as paths and trees, significantly improve vertex disambiguation, compared to the initial vertex features (see Figure \ref{fig:ablation} (left) and Table 2 in the appendix). As expected this allows for better fitting of the training data, which validates our claim that GNN expressivity improves. 

\medskip

\noindent\textbf{Generalisation. } However, none of the above claims can guarantee good generalisation in unseen data. For example, in Figure \ref{fig:ablation},  we observe that the test set performance does not follow the same trend with train performance when choosing substructures with strong vertex disambiguation. Aiming at better generalisation, it is desirable to make use of substructures for which there is prior knowledge of their importance in certain network distributions and have been observed to be intimately related to various properties. For example, small substructures (graphlets) have been extensively analysed in protein-protein interaction networks \cite{prvzulj2004modeling}, triangles and cliques characterise the structure of ego-nets and social networks in general \cite{Granovetter82thestrength}, simple cycles (rings) are central in molecular distributions, directed and temporal motifs have been shown to explain the working mechanisms of gene regulatory networks, biological neural networks, transportation networks and food webs \cite{milo2002network, DBLP:conf/wsdm/ParanjapeBL17, benson2016higher}. 

In Figure \ref{fig:ablation} (right), we showcase the importance of these inductive biases: a cycle-based GSN predicting molecular properties achieves smaller generalisation gap compared to a traditional MPNN, while at the same time generalising better with less training data. Choosing the best substructure collection is still an open problem that does not admit a straightforward solution due to its combinatorial nature. Alternatively, various heuristics can be used, e.g., motif frequencies or feature selection strategies.
Answering this question is left for future work.

\subsection{Complexity} 
The complexity of GSN comprises two parts: precomputation (substructure counting) and training/testing. The key appealing property is that training and inference are linear w.r.t the number of edges, $\mathcal{O}(|\mathcal{E}|)$, as opposed to higher-order methods \cite{maron2019provably, morris2019weisfeiler} with $\mathcal{O}(n^k)$, and \cite{vignac2020building} with $\mathcal{O}(n^2)$ training complexity and relational pooling \cite{DBLP:conf/icml/Murphy0R019} with $\mathcal{O}(n!)$ training complexity in absence of approximations. 

The worst-case complexity of subgraph isomorphism of fixed size $k$ is $\mathcal{O}(n^k)$, by examining all the possible $k$-tuples in the graph. However, for specific types of subgraphs, such as paths and cycles, the problem can be solved even faster (see e.g. \cite{DBLP:journals/algorithmica/GiscardKW19}). Approximate counting algorithms are also widely used, especially for counting frequent network motifs \cite{DBLP:journals/bioinformatics/KashtanIMA04, DBLP:conf/wabi/Wernicke05, DBLP:journals/tcbb/Wernicke06, DBLP:journals/bioinformatics/WernickeR06}, and can provide a considerable speed-up. Furthermore, recent neural approaches \cite{ying2020frequent, ying2020neural} provide fast approximate counting. 

In our experiments, we performed exact counting using the common isomorphism algorithm VF2 \cite{DBLP:journals/pami/CordellaFSV04}. Although its worst case complexity is $\mathcal{O}(n^k)$, it scales better in practice, for instance when the candidate subgraph is infrequently matched or when the graphs are sparse, and is also trivially parallelisable. In Figure \ref{fig:complexity}, we show a quantitative analysis of the empirical runtime of the counting algorithm against the worst case, for three different graph distributions: molecules, protein contact maps, social networks. It is easy to see that when the graphs are sparse (for the first two cases) and the number of matches is small, the algorithm is significantly faster than the worst case, while it scales better with the size of the graph $n$. Even, in the case of social networks, where several examples are near-complete graphs, both the runtime and the growth w.r.t both $n$ and $k$ are  better than the worst case. Overall, the preprocessing computational burden in most of the cases remains negligible for relatively small and sparse graphs, as it is the case of molecules. 

\section{Related Work}

\subsection{Expressive power of GNNs}

\noindent\textbf{WL hierarchy. }The seminal results in the theoretical analysis of the expressivity of GNNs \cite{xu2018how} and k-GNNs \cite{morris2019weisfeiler} established that traditional message passing-based GNNs are at most as powerful as the 1-WL test. \cite{chen2019equivalence} showed that graph isomorphism is equivalent to universal invariant function approximation. Higher-order Invariant Graph Networks (IGNs) have been studied in a series of works \cite{DBLP:conf/iclr/MaronBSL19, pmlr-v97-maron19a, maron2019provably, chen2019equivalence, keriven2019universal, puny2020graph}, establishing connections with the WL hierarchy, similarly to \cite{morris2019weisfeiler, MorrisNeurips2020}. The main drawbacks of these methods are the training and inference time complexity and memory requirements of $\mathcal{O}(n^k)$ and the super-exponential number of parameters (for linear IGNs) making them impractical, as well as their non-local nature making them more prone to overfitting. Finally, \cite{DBLP:Garg2020icml} also analysed the expressive power of MPNNs and other more powerful variants and provided generalisation bounds.

\medskip

\noindent\textbf{Unique identifiers. } From a different perspective, \cite{sato2019approximation} and \cite{Loukas2020What} showed the connections between GNNs and distributed local algorithms \cite{angluin1980local, linial1992locality, DBLP:conf/stoc/NaorS93} and suggested more powerful alternatives based on either local orderings or unique global identifiers (in the form of random features in \cite{sato2020random}) that make GNNs universal. Similarly, \cite{clip_ijcai20} propose to use random colorings in order to uniquely identify the nodes. However, these methods lack a principled permutation equivariant way to choose orderings/identifiers. To date this is an open problem in graph theory called \textit{graph canonisation} and it is at least as hard as solving graph isomorphism itself. A possible workaround is proposed in \cite{DBLP:conf/icml/Murphy0R019}, where the authors take into account all possible vertex permutations. However, obviously this quickly becomes intractable ($\mathcal{O}(n!)$) even when considering small-sized graphs. 

\medskip

\noindent\textbf{More expressive permutation equivariant GNNs. } 
Concurrently with our work, other more expressive GNNs have been proposed 
using equivariant message passing. In \cite{de2020natural}, the authors propose to linearly transform each message with a different kernel based on the local isomorphism class of the corresponding edge (similar to our definition of structural roles). However, as also noted by the authors, taking into account all possible local isomorphism classes leads to insufficient weight sharing and hence to overfitting. In contrast, in GSN, usually the substructure collection is small ($\mathtt{\sim}$5-10 graphs) and the substructures are repetitive in the graph distribution, and as a result generalisation improves. Vignac et al. \cite{vignac2020building} propose a message passing scheme where matrices of order equal to the size of the graph are propagated instead of vectors. This can be perceived as a practical unique identification scheme, but the neural network complexity becomes quadratic in the number of nodes. Finally, \cite{li2020distance} and \cite{beaini2020directional} enhance the aggregation function with
distance encodings (a strategy more relevant for vertex-level tasks) and graph eigenvectors respectively as alternative symmetry breaking mechanisms. In the experimental section, GSN is compared against these methods in real-world scenarios.
\medskip

\noindent\textbf{Quantifying expressivity. } 
Solely quantifying the  expressive power of GNNs in terms of their ability to distinguish non-isomorphic graphs does not provide the necessary granularity: even the 1-WL test can distinguish almost all (in the probabilistic sense) non-isomorphic graphs \cite{babai1980random}. As a result, there have been several efforts to analyse the power of $k$-WL tests in comparison to other graph invariants \cite{furer2010power, furer2017combinatorial, DBLP:conf/fct/ArvindFKV19, DBLP:conf/icalp/DellGR18}, while recently \cite{chen2020can} approached GNN expressivity by studying their ability to count substructures.

\subsection{Substructures in Complex Networks. }
The idea of analysing complex networks based on small-scale topological characteristics dates back to the 1970's and the notion of triad census for directed graphs \cite{holland1976local}. The seminal paper of \cite{milo2002network} coined the term \textit{network motifs} as over-represented subgraph patterns that were shown to characterise certain functional properties of complex networks in systems biology.  The closely related concept of  \textit{graphlets} \cite{prvzulj2004modeling, prvzulj2007biological, milenkovic2008uncovering, sarajlic2016graphlet}, different from motifs in being induced subgraphs, has been used to analyse the distribution of real-world networks and as a topological signature for network similarity. Our work is similar in spirit with the {\em graphlet degree vector} (GDV) \cite{prvzulj2007biological}, a node-wise descriptor based on graphlet counting. 

Substructures have been also used in the context of ML. In particular, subgraph patterns have been used to define Graph Kernels (GKs) \cite{horvath2004cyclic, shervashidze2009efficient, costa2010fast, DBLP:conf/icml/KriegeM12, nt2020graph}, with the most prominent being the graphlet kernel \cite{shervashidze2009efficient}. Motif-based node embeddings \cite{motif2vec, rossi-HONE-arxiv} and diffusion operators  \cite{motifnet, sankar2019meta, lee19-motif-attention}  that employ adjacency matrices weighted according to motif occurrences, have recently been proposed for graph representation learning. Our formulation provides a unifying framework for these methods and it is the first to analyse their expressive power. Finally, GNNs that operate in larger induced neighbourhoods \cite{li2019hierarchy, kim2019near} or higher-order paths \cite{flam2020neural} have prohibitive complexity since the size of these neighbourhoods typically grows exponentially. 
\section{Experimental Evaluation}\label{sec:experiments}

\begin{figure}[t]
        \centering
         \includegraphics[width=0.6\textwidth]{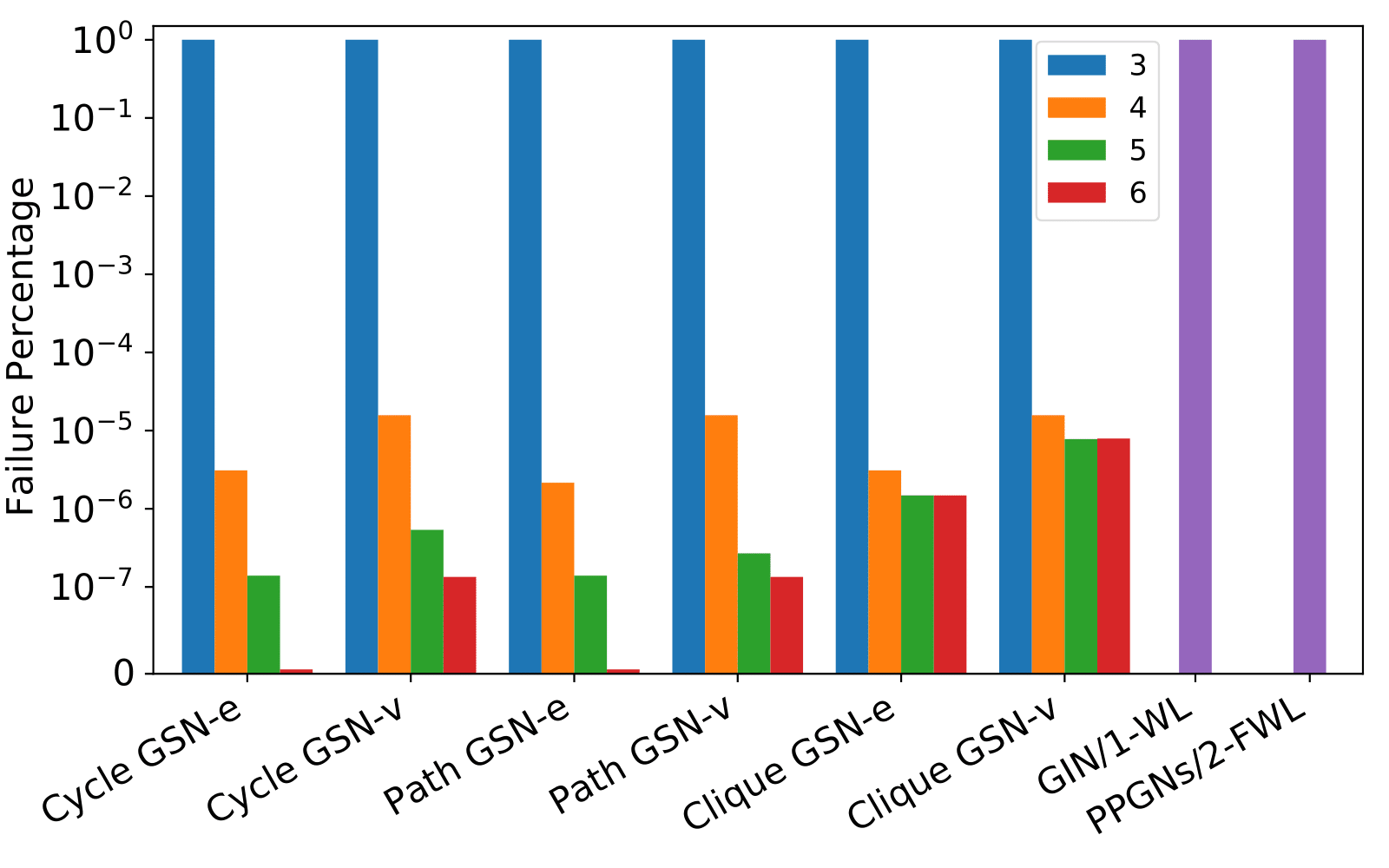}
        \caption{SR graphs isomorphism test (log scale, smaller values are better). Different colours indicate different substructure sizes.}
        \label{fig:sr_plot} 
\end{figure}

In the following section we evaluate \modelname\ in comparison to the state-of-the-art in a variety of datasets from different domains.  We are interested in practical scenarios where the collection of subgraphs, as well as their size, are kept small. Depending on the dataset domain we experimented with typical substructure families (\textit{cycles, paths} and \textit{cliques}) and maximum substructure size $k$ (note that for each setting, our substructure collection consists of all the substructures of the family with size $\leq k$). We also experimented with both graphlets and motifs and observed similar performance in most cases. To showcase that structural features can be used as an off-the-shelf strategy to boost GNN performance, we usually choose a base message passing architecture and minimally modify it into a GSN. Unless otherwise stated, the base architecture is a general-purpose MPNN with MLPs used in the message and update functions.  Additional implementation details can be found in the Appendix C.

\subsection{Synthetic Graph Isomorphism test}\label{sec:synthetic}
We tested the ability of GSNs to decide if two graphs are non-isomorphic on a collection of Strongly Regular graphs of size up to 35 nodes,  attempting to disambiguate pairs with the same number of nodes (for different sizes the problem becomes trivial). As we are only interested in the bias of the architecture itself, we use GSN with random weights to compute graph representations. Two graphs are deemed isomorphic if the Euclidean distance of their representations is smaller than a predefined threshold $\epsilon$. Figure \ref{fig:sr_plot} shows the failure percentage of our isomorphism test when using different graphlet substructures (\textit{cycles, paths}, and \textit{cliques}) of varying size $k$. Interestingly, the number of failure cases of GSN decreases rapidly as we increase $k$; cycles and paths of maximum length $k=6$ are enough to tell apart all the graphs in the dataset. Note that the performance of cliques saturates, possibly because the largest clique in our dataset has 5 nodes. Observe also the discrepancy between \modelname-v and \modelname-e. In particular, vertex-wise counts do not manage to distinguish all graphs, although missing only a few instances, which is in accordance with Theorem \ref{thrm: GSN-v vs GSN-e}. Finally, 1-WL \cite{xu2018how} and 2-FWL \cite{maron2019provably} equivalent models demonstrate 100\% failure, as expected from theory.

\subsection{TUD Graph Classification Benchmarks}\label{TUD}
We evaluate GSN on datasets from the classical TUD benchmarks. We use seven datasets from the  domains of bioinformatics and computational social science and compare against various GNNs and Graph Kernels. The base architecture that we used is GIN \cite{xu2018how}.
We follow the same evaluation protocol of \cite{xu2018how}, performing 10-fold cross-validation and then reporting the performance at the epoch with the best average accuracy across the 10 folds. Table \ref{tab:tud_datasets} lists all the methods evaluated with the split of \cite{zhang2018end}. We select our model by tuning architecture and optimisation hyperparameters and substructure related parameters, that is: (i) $k$, (ii) motifs against graphlets. Following domain evidence we choose the following substructure families: \textit{cycles} for molecules, \textit{cliques} for social networks. Best performing substructures both for GSN-e and GSN-v are reported. As can be seen, our model obtains state-of-the-art performance in most of the datasets, with a considerable margin against the main GNN baselines in some cases.

\begin{table*}[h]
    \caption{Graph classification accuracy on
    TUD Dataset.  {\bf \color{red} First}, {\bf \color{violet} Second}, {\bf \color{black} Third} best methods are highlighted. For GSN, the best performing structure is shown.
    $^*$Graph Kernel methods.  }
    \label{tab:tud_datasets}
    \centering
    \resizebox{\textwidth}{!}{
    \begin{tabular}{l cccc ccc}
        \toprule
        Dataset & MUTAG & PTC & Proteins & NCI1 & Collab & IMDB-B & IMDB-M \\
        \midrule        
        RWK* \cite{DBLP:conf/colt/GartnerFW03} & 
         79.2$\pm$2.1 & 
         55.9$\pm$0.3 & 
         59.6$\pm$0.1 & 
         $>$3 days & N/A & N/A &
         N/A\\
        
        GK* (k=3) \cite{shervashidze2009efficient} & 81.4$\pm$1.7& 
        55.7$\pm$0.5 &
        71.4$\pm$0.31 & 
        62.5$\pm$0.3 & 
        N/A & 
        N/A &
        N/A\\

        PK* \cite{DBLP:journals/ml/NeumannGBK16} & 
         76.0$\pm$2.7& 
         59.5$\pm$2.4 &
         73.7$\pm$0.7 & 
         82.5$\pm$0.5 & 
         N/A & 
         N/A & 
         N/A\\

          WL kernel* \cite{shervashidze2011weisfeiler} & 90.4$\pm$5.7 & 59.9$\pm$4.3 & 
         75.0$\pm$3.1 & \textcolor{red}{\textbf{86.0}$\pm$\textbf{1.8}} & 
         78.9$\pm$1.9 &
         73.8$\pm$3.9 &
         50.9$\pm$3.8 \\
         
        GNTK* \cite{DBLP:conf/nips/DuHSPWX19} & 
        90.0$\pm$8.5 & \textbf{\textcolor{violet}{67.9$\pm$6.9}} & 75.6$\pm$4.2 & \textbf{\textcolor{violet}{84.2$\pm$1.5}} & \textbf{\textcolor{violet}{83.6$\pm$1.0}} & \textbf{\textcolor{violet}{76.9$\pm$3.6}} & \textbf{\textcolor{violet}{52.8$\pm$4.6}}\\ 
         
        DCNN \cite{DBLP:conf/nips/AtwoodT16}& 
          N/A&  N/A &
          61.3$\pm$1.6
          & 56.6$\pm$1.0 &
          52.1$\pm$0.7 &
          49.1$\pm$1.4 &
          33.5$\pm$1.4\\

         DGCNN \cite{zhang2018end} & 85.8$\pm$1.8 & 
        58.6$\pm$2.5 & 
        75.5$\pm$0.9 & 
        74.4$\pm$0.5 & 
        73.8$\pm$0.5 &
        70.0$\pm$0.9 & 
        47.8$\pm$0.9\\

        IGN \cite{DBLP:conf/iclr/MaronBSL19} & 83.9$\pm$13.0 &
        58.5$\pm$6.9 &
        {\color{black} \textbf{76.6$\pm$5.5}} &
        74.3$\pm$2.7 & 
        78.3$\pm$2.5 & 
        72.0$\pm$5.5 & 
        48.7$\pm$3.4\\
        
        GIN \cite{xu2018how} & 
        89.4$\pm$5.6 & 
        64.6$\pm$7.0	& 
        76.2$\pm$2.8 &
        82.7$\pm$1.7 &
        80.2$\pm$1.9 & 
        75.1$\pm$5.1 &
        52.3$\pm$2.8\\

        PPGNs \cite{maron2019provably} &
        {\color{black} \textbf{90.6$\pm$8.7}} &
        66.2$\pm$6.6 &
        \textbf{\textcolor{violet}{77.2$\pm$4.7}} & 
        83.2$\pm$1.1 & 
        81.4$\pm$1.4 &
        73.0$\pm$5.8 & 
        50.5$\pm$3.6\\

        Natural GN \cite{de2020natural} &
        89.4$\pm$1.60 &
        66.8$\pm$1.79 &
        71.7$\pm$1.04 &
        82.7$\pm$1.35 &
        N/A &
        74.8$\pm$2.01 &
        51.3$\pm$1.50 \\

        WEGL \cite{kolouri2021wasserstein} &
        N/A &
        {\color{black} \textbf{67.5±7.7}} &
         76.5±4.2 &
         N/A &
        80.6±2.0 &
        75.4±5.0 &
        52.3±2.9\\

        GIN+GraphNorm \cite{cai2020graphnorm} &
        \textcolor{violet}{\textbf{91.6 ± 6.5}} &
        64.9 ± 7.5 &
       \textcolor{red}{\textbf{77.4 ± 4.9}} &
        82.7 ± 1.7&
         80.2 ± 1.0 &
        76.0 ± 3.7 &
        N/A\\

        \midrule
        {\bf GSN-e} & {\color{black}\textbf{90.6$\pm$7.5}}  &
        \textbf{\textcolor{red}{68.2$\pm$7.2}}& 
        {\color{black} \textbf{76.6$\pm$5.0}} & 
       {\color{black} \textbf{83.5$\pm$ 2.3}}&
        \textbf{\textcolor{red}{85.5$\pm$1.2}} &
        \textbf{\textcolor{red}{77.8$\pm$3.3}} & 
        \textbf{\textcolor{red}{54.3$\pm$3.3}}\\
         &  6 (cycles) 
         & 6 (cycles)  
         & 4 (cliques)
         & 15 (cycles) 
         & 3 (triangles)
         & 5 (cliques)  
         & 5 (cliques)   \\
        \midrule
        {\bf GSN-v}  &
        \textbf{\textcolor{red}{92.2$\pm$7.5}} & 
        67.4$\pm$5.7 & 
        74.59$\pm$5.0 &
        {\color{black}\textbf{ 83.5$\pm$2.0}} &
        {\color{black}\textbf{82.7$\pm$1.5}}  &
        {\color{black}\textbf{{76.8$\pm$2.0}}}& 
        {\color{black}\textbf{52.6$\pm$3.6}}\\
            &  12 (cycles) 
            &  10 (cycles) 
            &  4 (cliques)
            & 3 (triangles) 
            & 3 (triangles)
            &   4 (cliques) 
            &   3 (triangles) \\
        \bottomrule
    \end{tabular}
    }
\end{table*}

\subsection{ZINC Molecular graphs} We evaluate GSN on the task of regressing the ``penalized water-octanol
partition coefficient - logP'' (see \cite{gomez2018automatic, DBLP:conf/icml/KusnerPH17, DBLP:conf/icml/JinBJ18} for details) of molecules from the ZINC database \cite{DBLP:journals/jcisd/IrwinSMBC12, dwivedi2020benchmarking}. 
We use structural features obtained with $k$-cycle counting and report the result of the best performing substructure w.r.t. the validation set. 

As dictated by the evaluation protocol of \cite{dwivedi2020benchmarking}, the total number of parameters of the model is approximately 100K, which is achieved by selecting an appropriate network width.\footnote{A larger version of GSN using ~500K parameters attains \textbf{0.101 ± 0.010} test MAE.} The data split is obtained from \cite{dwivedi2020benchmarking} and the evaluation metric is the Mean Absolute Error (MAE). We compare against a variety of baselines, ranging from traditional message passing NNs to recent more expressive architectures \cite{corso2020principal, beaini2020directional, vignac2020building, balcilar2021breaking} and a molecular-specific one which is based on the junction tree molecular decomposition \cite{fey2020hierarchical}. Wherever possible, we compare two variants, one that does not take edge features into account and one that does. In both cases, GSN achieves state-of-the-art results outperforming all the baseline architectures.

\subsection{OGB-MOLHIV} 
We use \texttt{ogbg-molhiv} from the Open Graph Benchmark - OGB - \cite{DBLP:ogb} as a 
graph-level binary classification task, where the aim is to predict if a molecule inhibits HIV replication or not. We obtain baseline results from a plethora of different methods\footnote{the most representative ones from the OGB public leaderboard: \url{https://ogb.stanford.edu/docs/leader_graphprop/##ogbg-molhiv}}, ranging from algorithms specific to molecular graphs to general purpose architectures. Following the same rationale as in the previous experiments, we choose a base architecture and modify it into a GSN variant by introducing structural features in the aggregation function (cycle counts, similar to other molecular datasets). Here we use the following two base architectures: (a) \textit{GIN-VN}, a variation of GIN that allows for edge features and is extended with a \textit{virtual node}, i.e. a node connected to every node in the graph. 
(b) Directional Graph Networks (DGN), a GNN that propagates messages in an anisotropic manner, based on a predefined graph vector field. Observe that the vector field is an alternative way to break local symmetries. The authors of DGN use vector fields defined by the eigenvectors of the graph, while in our case the vector field is defined by graph substructures. More information can be found in the supplementary material. 

Using the evaluator provided by the authors, we report the ROC-AUC metric at the epoch with the best validation performance (substructures are also chosen based on the validation set).  By examining the results in Table~\ref{tab:molhiv-large} the following observations can be made,  (a) general purpose GNNs benefit from symmetry breaking mechanisms, either in the form of eigenvectors or in the form of substructures. (b) Cyclical substructures are a good inductive bias when learning on molecules (e.g. P-WL is a graph kernel based on topological features that contain information of graph cycles, similar to GSN). (c) Further evidence for that is provided by observing the performance of molecular fingerprints methods. In specific, a method based on the extended-connectivity fingerprints \cite{durant2002reoptimization}, which mainly focuses on the structure of the molecule reports 0.8060 ± 0.0010 test performance, while one that additionally uses the MACCS fingerprints \cite{rogers2010extended}, which mainly encode the presence of certain functional groups (i.e. both structure and attributes), reports 0.8232 ± 0.0047 test performance. These methods although not directly comparable to ours, 
currently achieve the best results in this dataset, thus this further motivates research on making GNNs structure-aware.

\begin{table}[t]
\begin{minipage}{0.39\textwidth}
    \centering
        \caption{MAE in \textit{ZINC}}
        \resizebox{1\textwidth}{!}{
          \begin{tabular}{l c c}
            \toprule
            Method & MAE & MAE (EF)\\
            \midrule
            GCN  \cite{DBLP:conf/iclr/KipfW17}& 0.469±0.002 & -\\
            GIN \cite{xu2018how} & 0.408±0.008 & -  \\
            GraphSage\cite{graphsage_neurips17} & 0.410±0.005 & - \\
            GAT \cite{DBLP:conf/iclr/VelickovicCCRLB18} 
            & 0.463±0.002 & -\\
            MoNet\cite{monti2017geometric} & 0.407±0.007 & - \\
            GatedGCN \cite{bresson2017residual} &
            0.422±0.006 & 0.363±0.009 \\
            
            MPNN & 0.254±0.014 & 0.209±0.018\\   
            
            MPNN-r & 0.322±0.026 & 0.279±0.023\\         
                        
            PNA\cite{corso2020principal} & 0.320±0.032 & 0.188±0.004\\

            DGN\cite{beaini2020directional} & \textbf{0.219±0.010} & 0.168±0.003\\
            
            GNNML\cite{balcilar2021breaking} & {\color{violet}\textbf{0.161±0.006}} & - \\

            HIMP\cite{fey2020hierarchical} & - & \textbf{0.151±0.006}\\
            
            SMP\cite{vignac2020building} & \textbf{0.219±} &{\color{violet}\textbf{0.138±}}\\

            \midrule

            \textbf{GSN} &{\color{red} \textbf{ 0.140$\pm$0.006}} & {\color{red} \textbf{0.115±0.012}}\\
            \bottomrule
          \end{tabular}%
          }
                  \label{tab:zinc_dataset}
        \end{minipage}
\hfill
\begin{minipage}{0.58\textwidth}
    \centering
    \caption{Test and Validation ROC-AUC in OGB-MOLHIV.}
    \resizebox{\textwidth}{!}{%
    \begin{tabular}{@{}lcc@{}}
    \toprule
    Method             & \vtop{\hbox{\strut Test}\hbox{\strut ROC-AUC }}                                    & \vtop{\hbox{\strut Validation}\hbox{\strut ROC-AUC }}                             \\ \midrule
    GIN+VN\cite{xu2018how}             & 0.7707 ± 0.0149                                 & \textbf{0.8479 ± 0.0068}                        \\
    DeeperGCN\cite{li2018deeper}          & 0.7858 ± 0.0117             & 0.8427 ± 0.0063                                 \\
    HIMP\cite{fey2020hierarchical}              & 0.7880 ± 0.0082                                 & -                                               \\
    GCN+GraphNorm\cite{cai2020graphnorm}      & 0.7883 ± 0.0100                                 & 0.7904 ± 0.0115                                 \\
    PNA\cite{corso2020principal}                & 0.7905 ± 0.0132                                 & {\color{violet} \textbf{0.8519 ± 0.0099}} \\
    PHC-GNN\cite{le2021parameterized}            & 0.7934 ± 0.0116                                 & 0.8217 ± 0.0089                                 \\
    DeeperGCN+FLAG\cite{kong2020flag}     &\textbf{0.7942 ± 0.0120}                                 & 0.8425 ± 0.0061                                 \\
    DGN + eigenvectors \cite{beaini2020directional} & {\color{violet} \textbf{0.7970 ± 0.0097}}                        & 0.8470 ± 0.0047                                 \\
    P-WL \cite{rieck2019persistent}              & {\color[HTML]{FE0000} \textbf{0.8039 ± 0.0040}} & 0.8279 ± 0.0059                                 \\ \midrule
    \textbf{GSN} (GIN+VN base)         & 0.7799±0.0100                                  & {\color[HTML]{FE0000} \textbf{0.8658±0.0084}} \\
    \textbf{GSN} (DGN + substructures) & {\color[HTML]{FE0000} \textbf{0.8039 ± 0.0090}} & 0.8473 ± 0.0096          
    \\ \bottomrule
    \end{tabular}%
    }
        \label{tab:molhiv-large}
    \end{minipage}
\end{table}

\subsection{Ablation Studies}

\subsubsection{Comparison between substructure collections}

In Figure \ref{fig:ablation} (left), we compare the training and test error for different substructure families (cycles, paths and non-isomorphic trees -- for each experiment we use all the substructures of size $\leq k$ in the family). Additionally, we measure the ``uniqueness'' of the identifiers each substructure yields as follows: for each graph $G$ in the dataset, we measure the number of unique vertex features $u_G$ (input vertex features concatenated with vertex structural identifiers for GSN-v). Then, we sum them up over the entire training set and divide by the total number of nodes, yielding the disambiguation score $\delta = \frac{\sum_{G} u_G}{\sum_{G} |V_G|}$. The disambiguation scores for the different types of substructures are illustrated as  horizontal bars in Figure \ref{fig:ablation} (the exact values can be found in Appendix C.2, Table 2). %

A first thing to notice is that the training error is tightly related to the disambiguation score. As identifiers become more discriminative, the model gains expressive power. On the other hand, the test error is not guaranteed to decrease when the identifiers become more discriminative. For example, although cycles have smaller disambiguation scores, they manage to generalise much better than the other substructures, the performance of which is similar to the baseline architecture (MPNN with MLPs). This is also observed when comparing against \cite{sato2020random} (\textit{MPNN-r} method in Table~\ref{tab:zinc_dataset}), where, akin to unique identifiers, random features are used to strengthen the expressivity of GNN architectures. This approach also fails to improve the baseline architecture in terms of the performance in the test set. This validates our intuition that unique identifiers can be hard to generalise when chosen in a non-permutation equivariant way and motivates once more the importance of choosing the identifiers not only based on their discriminative power, but also in a way that allows incorporating the appropriate inductive biases. Finally, we observe a substantial jump in performance when using GSN with cycles of size $k\geq 6$. This is not surprising, as cyclical patterns of such sizes (e.g. aromatic rings) are very common in organic molecules.

\begin{figure}[t]
\centering
\begin{subfigure}{0.49\textwidth}
  \centering
  \includegraphics[width=\textwidth]{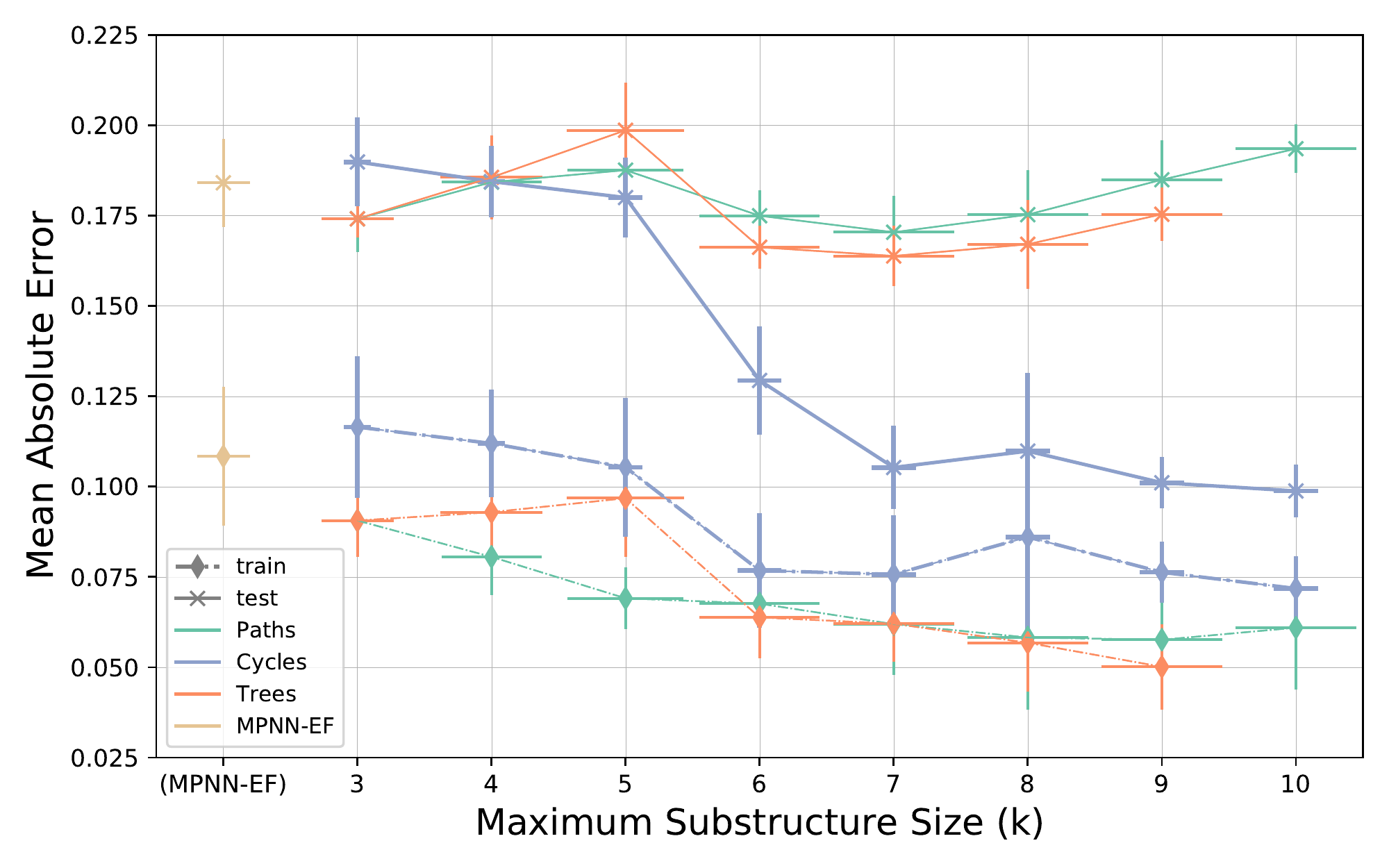}
\end{subfigure}
\begin{subfigure}{0.49\textwidth}
  \centering
     \includegraphics[width=\textwidth]{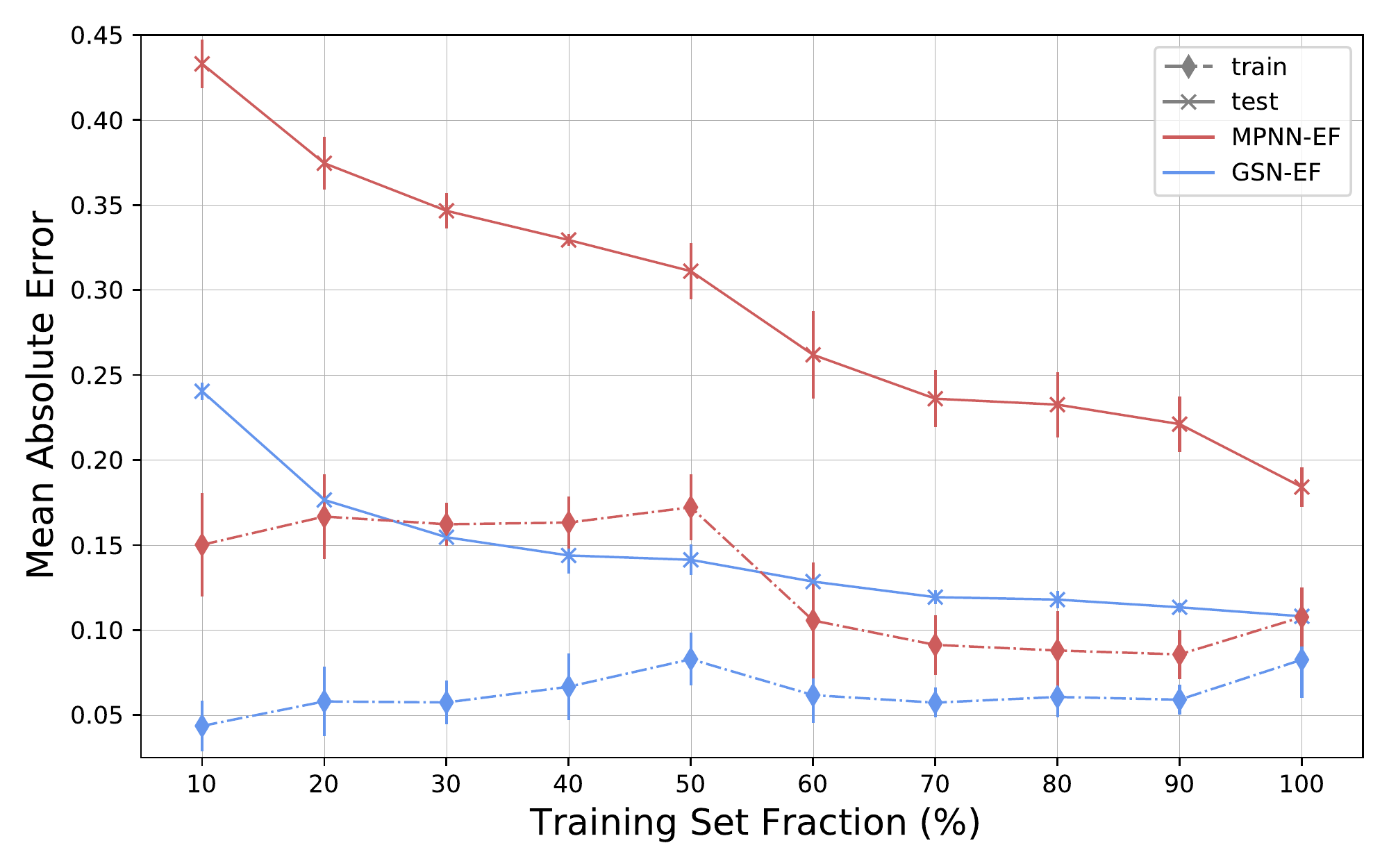}
\end{subfigure}
  \caption{(Left) Train (dashed) and test (solid) MAEs for Path-, Tree- and Cycle-GSN\textit{-EF} as a function of the maximum substructure size $k$. Vertical bars indicate standard deviation; horizontal bars depict disambiguation scores $\delta$. (Right) Train (dashed) and test (solid) MAEs for GSN\textit{-EF}(blue) and MPNN\textit{-EF}(red) as a function of the dataset fraction used for training}
  \label{fig:ablation}
\end{figure}

\subsubsection{Generalisation}
We repeat the experimental evaluation on ZINC using different fractions of the training set and compare the vanilla MPNN model against GSN. In Figure \ref{fig:ablation} (right), we plot the training and test errors of both methods. Regarding the training error, GSN consistently performs better, following our theoretical analysis on its expressive power. More importantly, GSN manages to generalise much better even with a small fraction of the training dataset. Observe that GSN requires only 20\% of the samples to achieve approximately the same test error that MPNN achieves when trained on the entire training set.

\subsubsection{Structural Features \& Message Passing:} We perform an ablation study on the abilities of the structural features to predict the task at  hand, when given as input to a graph-agnostic network. In particular, we compare our best performing GSN with a DeepSets model \cite{zaheer2017deep} that treats the input features and the structural identifiers as a set. For fairness of evaluation the same hyperparameter search is performed for both models (see Appendix C.5. Interestingly, as we show in Table \ref{tab:deepset_vs_gsn}, our baseline attains particularly strong performance across a variety of datasets and often outperforms other traditional message passing baselines. This demonstrates the benefits of these additional features and motivates their introduction in GNNs, which are unable to compute them. As expected, we observe that applying message passing on top of these features, brings performance improvements in the vast majority of the cases, sometimes considerably, as in the ZINC dataset.

\begin{table}[h]
\caption{Comparison between DeepSets and GSN with the same structural features}
\centering
\label{tab:deepset_vs_gsn}
\begin{tabular}{@{}lllll@{}}
\toprule
Dataset     & DeepSets              & \# params & GSN                       & \# params \\ \midrule
MUTAG       & \textbf{93.3$\pm$6.9} & 3K        & 92.8$\pm$7.0              & 3K        \\
PTC         & 66.4$\pm$6.7          & 2K        & \textbf{68.2$\pm$7.2}     & 3K        \\
Proteins    & \textbf{77.8$\pm$4.2} & 3K        & \textbf{77.8$\pm$5.6}     & 3K        \\
NCI1        & 80.3 $\pm$2.4         & 10K       & \textbf{83.5$\pm$ 2.0}    & 10K       \\
Collab      & 80.9 $\pm$1.6         & 30K       & \textbf{85.5$\pm$1.2}     & 52K       \\
IMDB-B      & 77.1 $\pm$3.7         & 51K       & \textbf{77.8$\pm$3.3}     & 65K       \\
IMDB-M      & 53.3 $\pm3.2$         & 68K       & \textbf{54.3$\pm$3.3}     & 66K       \\
ZINC        & 0.288 $\pm$0.003      & 366K      & \textbf{0.108 $\pm$0.018} & 385K      \\
ogbg-molhiv & 77.34$\pm$1.46        & 3.4M      & \textbf{77.99$\pm$1.00}   & 3.3M      \\ \bottomrule
\end{tabular}
\end{table}

\section{Conclusion}

In this paper, we propose a novel way to design structure-aware graph neural networks. Motivated by the limitations of traditional GNNs to capture important topological properties of the graph, we formulate a message passing scheme enhanced with structural features that are extracted by subgraph isomorphism.
We show both theoretically and empirically that our construction leads to improved expressive power and attains state-of-the-art performance in real-world scenarios. In future work, we will further explore the expressivity of GSNs as an alternative to the $k$-WL tests, as well as their generalisation capabilities. Another important direction is to infer prominent substructures directly from the data and explore the ability of graph neural networks to compose substructures.

\section{Acknowledgements}
\label{sec:acks}
This research was partially supported by the ERC Consolidator Grant No. 724228 - LEMAN (GB and MB). The work of GB is partially funded by a PhD scholarship from the Department of Computing, Imperial College London. 
SZ acknowledges support from the EPSRC Fellowship DEFORM: Large Scale Shape Analysis of Deformable Models of Humans (EP/S010203/1) and a Google Faculty award. MB acknowledges support from Google Faculty awards and the Royal Society Wolfson Research Merit award.

{\small
\bibliographystyle{unsrt}
\bibliography{egbib}
}

\newpage
\appendix
\section{Deferred Proofs} 

\subsection{GSN is permutation equivariant}\label{proof:perm_equiv}

\begin{proof}
Let $\mathbf{A} \in \mathbb{R}^{n\times n}$ the adjacency matrix of the graph, $\mathbf{H}_0 \in \mathbb{R}^{n\times d^V_{in}}$ the input vertex features, $\mathbf{E} \in \mathbb{R}^{n\times n \times d^E_{in}}$ the input edge features and $\mathbf{E}_i \in \mathbb{R}^{n\times n}$ the edge features at dimension $i$. Let $S^V(\mathbf{A}) \in \mathbb{N}^{n\times d_V}, S^E(\mathbf{A}) \in \mathbb{N}^{n\times n \times d_E}$ the functions generating the vertex and edge structural identifiers respectively. 

It is obvious that subgraph isomorphism is invariant to the ordering of the vertices, i.e. we will always obtain the same matching between subgraphs $G_S$ and graphs $H$ in the subgraph collection. Thus, each vertex $v$ (edge $(v,u)$) in the graph will be assigned the same structural identifiers $\mathbf{x}^V_v$ ($\mathbf{x}^E_{v,u}$) regarless of the vertex ordering, and $S^V$ and $S^E$ are permutation equivariant, i.e., for any permutation matrix $\mathbf{P} \in \{0,1\}^{n\times n}$ it holds that 
\begin{equation*}
    \begin{split}
S^V(\mathbf{P}\mathbf{A}\mathbf{P}^T) &= \mathbf{P}S^V(\mathbf{A})\\
S^E(\mathbf{P}\mathbf{A}\mathbf{P}^T) &= \mathbf{P}[S^E_1(\mathbf{A});\dots ;S^E_{d_E}(\mathbf{A})]\mathbf{P}^T,
    \end{split}
\end{equation*}
where the permutation is applied at each slice $S^E_i(\mathbf{A}) \in \mathbb{N}^{n\times n}$ of the tensor $S^E(\mathbf{A})$.

Let $f(\mathbf{A}, \mathbf{H}, \mathbf{E}) \in \mathbb{R}^{n\times d_{out}}$ a GSN layer. We will show that $f$ is permutation equivariant.  We need to show that if $\mathbf{Y} = f(\mathbf{A}, \mathbf{H}, \mathbf{E})$ the output of the GSN layer, then $\mathbf{PY} = f(\mathbf{PAP}^T, \mathbf{PH}, \mathbf{P}[\mathbf{E}_1;  \dots ;\mathbf{E}_{d_E}] \mathbf{P}^T)$ for any permutation matrix $\mathbf{P}$. It is easy to see that GSN-v can can be expressed as a traditional MPNN $g(\mathbf{A}, \mathbf{H}, \mathbf{E}) \in \mathbb{R}^{n\times d_{out}}$ (similar to Eq. (5) of the main paper) by replacing the vertex features $\mathbf{H}$ with the concatenation of the input vertex features and the vertex structural identifiers, i.e. $\mathbf{Y} = f(\mathbf{A}, \mathbf{H}, \mathbf{E}) = g(\mathbf{A}, [\mathbf{H}; S^V(\mathbf{A})], \mathbf{E})$. Thus, 

\begin{align*}
    &f(\mathbf{PAP}^T, \mathbf{P}\mathbf{H}, \mathbf{P}[\mathbf{E}_1;\dots;\mathbf{E}_{d_E} ] \mathbf{P}^T, ) =
    \\&= g(\mathbf{PAP}^T, [\mathbf{PH}; S^V(\mathbf{P}\mathbf{A}\mathbf{P}^T)], \mathbf{P}[\mathbf{E}_1 ; \dots; \mathbf{E}_{d_E}] \mathbf{P}^T)\\
    & = g(\mathbf{PAP}^T, \mathbf{P}[\mathbf{H}; S^V(\mathbf{A})], \mathbf{P}[\mathbf{E}_1 ; \dots; \mathbf{E}_{d_E}] \mathbf{P}^T) \\
    &= \mathbf{P} g(\mathbf{A}, [\mathbf{H}; S^V(\mathbf{A})], \mathbf{E})\\
    & = \mathbf{PY}\\
\end{align*}
where in the last step we used the permutation equivariant property of MPNNs. Similarly, we can show that a GSN-e layer is permutation equivariant, since it can be expressed as a traditional MPNN layer by replacing the edge features with the concatenation of the original edge features and the edge structural identifiers $f(\mathbf{A}, \mathbf{H}, \mathbf{E}) = g(\mathbf{A}, \mathbf{H}, [\mathbf{E};S^E(A)])$.

Overall, a GSN network is permutation equivariant as composition of permutation equivariant functions, or permutation invariant when composed with a permutation invariant layer at the end, i.e. the READOUT function.

\end{proof}

\subsection{Proof of Theorem 3.1: {\modelname\ is at least as powerful as the 1-WL test}}\label{sec:proof_1}

\begin{proof}

To show that GSN it is at least as expressive as the 1-WL test, we will repurpose the proof of Theorem 3 in \cite{xu2018how}  and demand the injectivity of the update function (w.r.t. both the hidden state $\mathbf{h}^t_v$ and the message $\mathbf{m}^{t+1}_v$), and the injectivity of the aggregation w.r.t. the multiset of the hidden states of the neighbours $\Lbag \mathbf{h}^{t}_u\Rbag_{u\in\mathcal{N}(v)}$. It suffices then to show that if injectivity is preserved then \modelname s are at least as powerful as the 1-WL.

We will show the above statement for vertex-labelled graphs, since traditionally the 1-WL test does not take into account edge labels.\footnote{if one considers a simple 1-WL extension that concatenates edge labels to neighbour colours, then the same proof applies.} We can rephrase the statement as follows: \textit{If GSN deems two graphs $G_1$, $G_2$ as isomorphic, then also 1-WL deems them isomorphic}. Given that the graph-level representation is extracted by a readout function that receives the multiset of the vertex colours in its input (i.e. the graph-level representation is the vertex colour histogram at some iteration $t$), then it suffices to show that if for the two graphs the multiset of the vertex colours that GSN infers is the same, then also 1-WL will infer the same multiset for the two graphs. 

Consider the case where the two multisets that GSN extracts are the same:  i.e. ${\Lbag \mathbf{h}^{t}_v\Rbag_{v\in\mathcal{V}_{G_1}} = \Lbag \mathbf{h}^{t}_u\Rbag_{u\in\mathcal{V}_{G_2}}}$. Then both multisets contain the same distinct colour/hidden representations with the exact same multiplicity. Thus, it further suffices to show that
if two vertices $v, u$ (that may belong to the same or to different graphs) have the same GSN hidden representations $\mathbf{h}^t_v = \mathbf{h}^t_u$ at any iteration $t$, then they will also have the same colours $c^t_v = c^t_u$, extracted by 1-WL. Intuitively, this means that GSN creates a partition of the vertices of each graph that is at least as fine-grained as the one created by 1-WL. We prove by induction (similarly to \cite{xu2018how}) that GSN model class contains a model where this holds (w.l.o.g. we show that for GSN-v; same proof applies to GSN-e).

For $t=0$ the statement holds since the initial vertex features are the same for both GSN and 1-WL, i.e. $\mathbf{h}^0_v = c^0_v, \quad \forall v \in \mathcal{V}_{G_1}\cup\mathcal{V}_{G_2}$. Suppose the statement holds for $t-1$, i.e. ${\mathbf{h}^{t-1}_v = \mathbf{h}^{t-1}_u \Rightarrow c^{t-1}_v = c^{t-1}_u}$. Then we show that it also holds for $t$. Every vertex hidden representation at step $t$ is updated as follows: ${\mathbf{h}^{t}_v  = \mathrm{UP}^{t}\big(\mathbf{h}^{t-1}_v, \mathbf{m}^{t}_v\big)}$. Assuming that the update function $\mathrm{UP}^{t}$ is injective, we have the following: if $\mathbf{h}^{t}_v =  \mathbf{h}^{t}_u$, then:
\begin{itemize}[wide, labelwidth=!, labelindent=0pt]
\item $\mathbf{h}^{t-1}_v = \mathbf{h}^{t-1}_u$, which from the induction hypothesis implies that  $c^{t-1}_v = c^{t-1}_u.$
\item $\mathbf{m}^{t}_v = \mathbf{m}^{t}_u$, where the message function is defined as in Eq. (5) of the main paper.
Additionally here we require $M^{t}$ to be injective w.r.t. the multiset of the hidden representations of the neighbours:\footnote{Lemma 5 from \cite{xu2018how} states that such a function always exists assuming that the elements of the multiset originate from a countable domain}

\begin{align*}
    &\mathbf{m}^{t}_v = \mathbf{m}^{t}_u \Rightarrow 
    \Lbag \mathbf{h}^{t-1}_w \Rbag_{w\in\mathcal{N}_v} = \Lbag \mathbf{h}^{t-1}_z\Rbag_{z\in\mathcal{N}_u}
\end{align*}

From the induction hypothesis we know that  ${\mathbf{h}^{t-1}_w = \mathbf{h}^{t-1}_z}$ implies that ${c^{t-1}_w = c^{t-1}_z}$ for any  ${w\in\mathcal{N}_v, z\in\mathcal{N}_u}$, thus ${\Lbag c^{t-1}_w\Rbag_{w\in\mathcal{N}_v} = \Lbag c^{t-1}_z\Rbag_{z\in\mathcal{N}_u}}$.
\end{itemize}
Concluding, given the update rule of 1-WL: ${c^{t}_v = \mathrm{HASH}\Big(c^{t-1}_v, \  \Lbag c^{t-1}_w \Rbag_{w\in\mathcal{N}_v}\Big)}$, it holds that $c^{t}_v = c^{t}_u.$

\end{proof}

\subsection{Proof of Corollary 3.2}\label{sec:proof_universality}
\begin{proof}
In order to prove the universality of GSN, we will show that when the substructure collection contains all graphs of size $n-1$, then there exists a parametrisation of GSN that can infer the isomorphism classes of all vertex-deleted subgraphs of the graph $G$ (the \textit{deck} of $G$). 

The reconstruction conjectures states that two graphs with at least three vertices are isomorphic if and only if they have the same deck. Thus, the deck is sufficient to distinguish all non-isomorphic graphs. The deck can be defined as follows:

Let $\mathcal{H} = \{H_1, H_2, \cdots, H_K\}$ the set of all possible graphs of size $n-1$. The vertex-deleted subgraphs of $G$ are by definition all the induced subgraphs of $G$ with size $n-1$, which we denote as
\begin{equation*}
\mathcal{G}_{n-1} = \{G_S \text{: induced subgraph of } G \text{ with }|\mathcal{V}_{G_S}| = n-1\}.
\end{equation*}
Then, the deck $D(G)$ can be defined as a vector of size $|\mathcal{H}|$, where at the $j$-th dimension
\begin{equation*}
\begin{split}
   D_j(G) &= \bigg|\Big\{G_S \in \mathcal{G}_{n-1}\ | \ G_S\simeq H_j \Big\}\bigg|\ \\
   &= \sum_{G_S \in \mathcal{G}_{n-1}}\mathds{1}[G_S\simeq H_j],
\end{split}
\end{equation*}
where $\mathds{1}[\cdot]$ the indicator function. 
The structural feature  $x_{H_j,i}^V (v)$ for each substructure $H_j$ and orbit $i$ are computed as follows:
\begin{align*}
    x_{H_j, i}^V (v) &= \bigg| \Big\{G_S \simeq H_j \ | \, v \in \mathcal{V}_{G_S}, \  f(v) \in O_{H_j,i}^V \Big\} \bigg|\\
    &= \sum_{G_S \in \mathcal{G}_{n-1}}\mathds{1}[v \in \mathcal{V}_{G_S}]\mathds{1}[G_S \simeq H_j]\mathds{1}[f_{G_S}(v) \in O_{H_j,i}^V]
\end{align*}
where $f_{G_S}(\cdot) = \emptyset$ if $G_S \not \simeq H_j$, otherwise it is the bijective mapping function. The deck can be inferred as follows:
\begin{align*}
    \sum_{v\in \mathcal{V}_G}\sum_{i=1}^{d_{H_j}} x_{H_j, i}^V (v) 
    & =  \sum_{v\in \mathcal{V}_G}\sum_{i=1}^{d_{H_j}}\sum_{G_S \in \mathcal{G}_{n-1}}\mathds{1}[v \in \mathcal{V}_{G_S}]\mathds{1}[G_S \simeq H_j]\mathds{1}[f_{G_S}(v) \in O_{H_j,i}^V]\\
    & = \sum_{v\in \mathcal{V}_G}\sum_{G_S \in \mathcal{G}_{n-1}}\mathds{1}[v \in \mathcal{V}_{G_S}]\mathds{1}[G_S \simeq H_j]\sum_{i=1}^{d_{H_j}}\mathds{1}[f_{G_S}(v) \in O_{H_j,i}^V]\\
    & = \sum_{G_S \in \mathcal{G}_{n-1}}\mathds{1}[G_S \simeq H_j]\sum_{v\in \mathcal{V}_G}\mathds{1}[v \in \mathcal{V}_{G_S}]\\
    & =\sum_{G_S \in \mathcal{G}_{n-1}} (n-1)\mathds{1}[G_S \simeq H_j]\\
    & = (n-1)D_j(G)\\    
\end{align*}
where we used that $\sum_{i=1}^{d_{H_j}}\mathds{1}[f_{G_S}(v) \in O_{H_j,i}^V]$ = 1, since each vertex can be mapped to a single orbit only. Thus, the deck can be inferred by a simple GSN-v parametrisation (a linear layer with depth equal to $|\mathcal{H}|$ that performs orbit-wise summation and division by the constant $n-1$ for each vertex separately, followed by a sum readout). Since GSN-v can be inferred by GSN-e (Theorem 3.3), then GSN-e is also universal.
\end{proof}

\subsection{Proof of Theorem 3.3}\label{proof_GSNv_GSNe}

\begin{proof}
Without loss of generality we will show Theorem 3.3 for the case of a single substructure $H$. In order to show that GSN-e can express GSN-v, we will first prove the following: \textit{the vertex identifier of a vertex $v$ can be inferred by the edge identifiers of its incident edges}.

To simplify notation we define the following: For an orbit $\mathrm{Orb}^V(v)$, denote the \textit{orbit neighbourhood} as the multiset of the orbits of the neighbours of $v$ and the \textit{orbit degree} as the degree of any vertex in $\mathrm{Orb}^V(v)$  :

\begin{equation*}
\mathcal{N}\big(\mathrm{Orb}^V(v)\big) = \Lbag \mathrm{Orb}^V(u) \ | \ u \in \mathcal{N}(v)\Rbag \ \text{ and } \ \text{deg}(v) = \text{deg}(\mathrm{Orb}^V(v)) = |\mathcal{N}\big(\mathrm{Orb}^V(v)\big)|.
\end{equation*}
For brevity we will use the following notation: vertex orbits are indexed as follows $O^V_{1}, O^V_{2}, \ldots O^V_{d_V} $ and edge orbits $O_{1,1}^E, O_{1,2}^E, \ldots O_{d_V, d_V}^E$ with $O_{i,j}^E = \big( O^V_{i},  O^V_{j}\big)$. Structural features are denoted accordingly: $x_{i}^V(v)$ and $x_{ij}^E(v,u)$

Let's assume that there exists only one matched subgraph $G_S \simeq H$ and the bijection between $\mathcal{V}_{G_S}$ and $\mathcal{H}$ is denoted as $f$. Then, for an abitrary vertex $v$ and a vertex orbit $O^V_i$, one of the following holds:
\begin{itemize}
    \item $v \not \in \mathcal{V}_{G_S}$. Then $x_{i}^V (v) = 0$ and $x_{ij}^E (v,u) = 0, \ \forall u \in \mathcal{N}(v)$,
    \item $v \in \mathcal{V}_{G_S}$ and $\mathrm{Orb}(f(v)) \neq O^V_i$. Then, $x_{i}^V (v) = 0$ and $x_{ij}^E (v,u) = 0, \ \forall u \in \mathcal{N}(v)$. %
    Note that here the directionality of the edge is important, otherwise we cannot determine the value of $x_{ij}^E (v,u)$ unless we also know the orbit of $f(u)$.
    \item $v \in \mathcal{V}_{G_S}$ and $\mathrm{Orb}(f(v)) = O^V_i$. Then, $x_{i}^V (v) = 1$ and since $f(v)$ has exactly $\text{deg}(O^V_i)$ neighbours in $H$, then $v$ has exactly $\text{deg}(O^V_i)$ neighbours in $G_S$ with vertex orbits $\mathcal{N}\big(O^V_i\big)$. In other words
    \begin{equation*}
        \sum_{u \in \mathcal{N}(v)} \sum_{j:O^V_j \in \mathcal{N}(O^V_{i})}  x_{ij}^E (u, v) = \text{deg}(O^V_i)
    \end{equation*}
\end{itemize}

Thus, by induction, for $m$ matched subgraphs $G_S \simeq H$ with $v \in \mathcal{V}_{G_S}$ and $\mathrm{Orb}(f(v)) = O^V_i$, it holds that $x_{i}^V (v) = m$ and $\sum_{u \in \mathcal{N}(v)} \sum_{j:O^V_j \in \mathcal{N}(O^V_{i})}  x_{ij}^E (u, v) = m*\text{deg}(O^V_i)$. Then it follows that:

\begin{equation}\label{gsn_v_gsn_e}
        x_{i}^V (v)  = \frac{1}{\text{deg}(O^V_{i})}\sum_{u \in \mathcal{N}(v)} \sum_{j:O^V_j \in \mathcal{N}(O^V_{i})}  x_{ij}^E (u, v)
\end{equation}

The rest of the proof is straightforward: we will assume a GSN-v using substructure counts of the graph $H$, with $L$ layers and width $w$ defined as in the main paper (Eq. (5))
Then, there exists a GSN-e with $L+1$ layers, where the first layer has width $d^{in}_V + d_V$ and implements the following function: $\mathrm{UP}^{t+1}\big(\mathbf{h}^{t}_v, \mathbf{m}^{t+1}_v\big) = [\mathbf{h}^{t}_v ; \mathbf{m}^{t+1}_v]$, where:
\begin{align*}
    \mathbf{m}^{t+1}_v  &= M^{t+1}\bigg(\Lbag(\mathbf{h}^{t}_v, \mathbf{h}^{t}_u, \mathbf{x}^E_{v,u},  \mathbf{e}_{u,v})\Rbag_{u\in \mathcal{N}(v)}  \bigg)\\
     &= [\frac{1}{\text{deg}(O^V_{1})}\sum_{u \in \mathcal{N}(v)} \sum_{j:O^V_j \in \mathcal{N}(O^V_{1})}  x_{1j}^E (u, v); \ \dots \ ;\\
     &\quad \frac{1}{\text{deg}(O^V_{d_V})}\sum_{u \in \mathcal{N}(v)} \sum_{j:O^V_j \in \mathcal{N}(O^V_{d_V})}  x_{d_V,j}^E (u, v)]\\
     &=\mathbf{x}^V_v
\end{align*}
Note that since $M$ is a universal multiset function approximator, then there exist parameters of $M$ with which the above function can be computed. The next $L$ layers of GSN-e can implement a traditional MPNN where now the input vertex features are $[\mathbf{h}^{t}_v ; \mathbf{x}^V_v]$ (which is exactly the formulation of GSN-v) and this concludes the proof. 

\end{proof}

\section{Comparison with higher-order Weisfeiler-Leman tests}\label{k-wl-comparison}

\subsection{The WL hierarchy}\label{k-wl}

Following the terminology introduced in \cite{maron2019provably}, we describe the so-called {\em Folklore WL family ($k$-FWL)}. Note that, in the majority of papers on GNN expressivity \cite{morris2019weisfeiler, maron2019provably, chen2020can} another family of WL tests is discussed, under the terminology \textit{$k$-WL} with expressive power equal to $(k-1)$-FWL. In contrast, in most graph theory papers on graph isomorphism \cite{DBLP:journals/combinatorica/CaiFI92, furer2017combinatorial, DBLP:conf/fct/ArvindFKV19} the $k$-WL term is used to describe the algorithms referred to as $k$-FWL in GNN papers. Here, we follow the $k$-FWL convention to align with the work mostly related to ours.

The $k$-FWL operates on $k$-tuples of vertices ${\mathbf{v}=(v_1, v_2, \dots ,v_k)}$ to which an initial colour $c^0_{\mathbf{v}}$ is assigned based on their \textit{isomorphism types} (see section \ref{sr_graphs_wl}), which can loosely be thought of as a generalisation of isomorphism that also preserves the ordering of the vertices in the tuple. Then, at each iteration the colour is refined as follows:
\begin{equation}
    {c^{t+1}_\mathbf{v} = 
\mathrm{HASH} \Big( c^{t}_\mathbf{v}, \, \Lbag \big( c^{t}_{\mathbf{v}_{u,1}},c^{t}_{\mathbf{v}_{u,2}},\dots ,c^{t}_{\mathbf{v}_{u,k}} \big) \Rbag_{u\in\mathcal{V}} \Big)},
\end{equation}
where $\mathbf{v}_{u,j} = (v_1, v_2, \dots ,v_{j-1}, u, v_{j+1}, \dots ,v_k)$. 

The multiset $\Lbag \big( c^{t}_{\mathbf{v}_{u,1}},c^{t}_{\mathbf{v}_{u,2}},\dots ,c^{t}_{\mathbf{v}_{u,k}} \big) \Rbag_{u\in\mathcal{V}}$ can be perceived as a form of generalised neighbourhood. Observe that all possible tuples in the graph store information necessary for the updates, thus each $k$-tuple receives information {\em from the entire graph}, contrary to the {\em local} nature of the 1-WL test.

\subsection{Why does 2-FWL fail on strongly regular graphs?}\label{sr_graphs_wl}

\begin{proof}

We first formally define what an isomorphism type is and what are the properties of the SR family:

\begin{definition}[Isomorphism Types] 
Two k-tuples $\mathbf{v}^a=\{v_1^a, v_2^a, \dots, v_k^a\}$, 
$\mathbf{v}^b=\{v_1^b, v_2^b, \dots, v_k^b\}$ will have the same isomorphism type iff:

\begin{itemize}
    \item  $\forall \, i,j \in \{0,1, \dots, k \}, \quad v_i^a=v_j^a \Leftrightarrow v_i^b=v_j^b$
    \item  $\forall \, i,j \in \{0,1, \dots, k \},     \quad v_i^a \sim v_j^a \Leftrightarrow v_i^b \sim v_j^b$, where $\sim$ means that the vertices are adjacent.
\end{itemize}
\end{definition}

Note that this is a stronger condition than isomorphism, since the mapping between the vertices of the two tuples needs to preserve order. In case the graph is employed with edge and vertex features, these need to be preserved as well (see \cite{chen2020can}) for the extended case). 

\begin{definition}[Strongly regular graph] 
A {\em SR($n$,$d$,$\lambda$,$\mu$)-graph} is a regular graph with $n$ vertices and degree $d$, where every two adjacent vertices have always $\lambda$ mutual neighbours, while every two non-adjacent vertices have always $\mu$ mutual neighbours. 
\end{definition}

Now we can proceed to the details of the proof. For the 2-FWL test, when working with simple undirected graphs without self-loops, we have the following 2-tuple isomorphism types:
\begin{itemize}
    \item $\mathbf{v}=\{v_1, v_1\}$: \textit{vertex type}. Mapped to the colour $c^0 = c_\alpha$
    \item $\mathbf{v}=\{v_1, v_2\}$ and $v_1 \not\sim v_2$:\textit{ non-edge type}. Mapped to the colour $c^0 = c_\beta$
    \item $\mathbf{v}=\{v_1, v_2\}$ and $v_1 \sim v_2$: \textit{edge type}. Mapped to the colour $c^0 = c_\gamma$
\end{itemize}
 
 For each 2-tuple $\mathbf{v}=\{v_1, v_2\}$, a generalised ``neighbour'' is the following tuple: ${(\mathbf{v}_{u,1}, \mathbf{v}_{u,2}) = \big((u, v_2), (v_1, u)\big)}$, where $u$ is an arbitrary vertex in the graph. 

Now, let us consider a strongly regular graph SR($n$,$d$,$\lambda$,$\mu$). We have the following cases:

\begin{itemize}
    \item generalised neighbour of a  \textit{vertex type} tuple: ${(\mathbf{v}_{u,1}, \mathbf{v}_{u,2}) = \big((u, v_1), (v_1, u)\big)}$. The corresponding neighbour colour tuples are:
    \begin{itemize}
        \item  $(c_\alpha, c_\alpha)$ if  $v_1=u$,
        \item $(c_\beta, c_\beta)$ if $v_1 \not \sim u$ ,
        \item $(c_\gamma, c_\gamma)$ if $v_1 \sim u$.
    \end{itemize}
    The update of the 2-FWL is: $c^{1}_{\mathbf{v}} = 
\mathrm{HASH} \Big( c_\alpha, \  \Lbag \underbrace{(c_\alpha, c_\alpha)}_{\text{1 time}}, \underbrace{(c_\beta, c_\beta)}_{\text{ $n-1-d$ times}}, \underbrace{(c_\gamma, c_\gamma)}_{\text{ $d$ times}} \Rbag\Big)$ same for all \textit{vertex type} 2-tuples.
    
    \item generalised neighbour of a \textit{non-edge type} tuple: ${(\mathbf{v}_{u,1}, \mathbf{v}_{u,2}) = \big((u, v_2), (v_1, u)\big)}$. The corresponding neighbour colour tuples are:
    \begin{itemize}
        \item $(c_\alpha, c_\beta)$ if $v_2=u$,  
        \item $(c_\beta, c_\alpha)$ if $v_1=u$, 
        \item $(c_\gamma, c_\beta)$ if $v_2 \sim u$ and $v_1 \not \sim u$, 
        \item $(c_\beta, c_\gamma)$ if $v_1 \sim u$ and $v_2 \not \sim u$, 
        \item $(c_\beta, c_\beta)$ if $v_1 \not \sim u$ and $v_2 \not \sim u$, 
        \item $(c_\gamma, c_\gamma)$ if $v_1 \sim u$ and $v_2 \sim u$.
    \end{itemize}
     The update of the 2-FWL is:
     
     $c^{1}_{\mathbf{v}} = 
\mathrm{HASH} \Big(c_\beta, \ \Lbag \underbrace{(c_\alpha, c_\beta)}_{\text{1 time}}, \underbrace{(c_\beta, c_\alpha)}_{\text{1 time}}, \underbrace{(c_\gamma, c_\beta)}_{\text{$d-\mu$ times}}, \underbrace{(c_\beta, c_\gamma)}_{\text{$d-\mu$ times}}, \underbrace{(c_\beta, c_\beta)}_{\text{$n-2 - (2d - \mu)$ times}}, \underbrace{(c_\gamma, c_\gamma)}_{\text{ $\mu$ times}} \Rbag \Big)$
same for all \textit{non-edge type} 2-tuples.

    \item generalised neighbour of an \textit{edge type} tuple:  
    \begin{itemize}
        \item $(c_\alpha, c_\gamma)$ if $v_2=u$, 
        \item $(c_\gamma, c_\alpha)$ if $v_1=u$, 
        \item $(c_\gamma, c_\beta)$ if $v_2 \sim u$ and $v_1 \not \sim u$, 
        \item $(c_\beta, c_\gamma)$ if $v_1 \sim u$ and $v_2 \not \sim u$, 
        \item $(c_\beta, c_\beta)$ if $v_1 \not \sim u$ and $v_2 \not \sim u$, 
        \item $(c_\gamma, c_\gamma)$ if $v_1 \sim u$ and $v_2 \sim u$.
    \end{itemize}
    The update of the 2-FWL is:
     
     $c^{1}_{\mathbf{v}} = 
\mathrm{HASH} \Big( c_\gamma, \ \Lbag \underbrace{(c_\alpha, c_\gamma)}_{\text{1 time}}, \underbrace{(c_\gamma, c_\alpha)}_{\text{1 time}}, \underbrace{(c_\gamma, c_\beta)}_{\text{$d-\lambda$ times}}, \underbrace{(c_\beta, c_\gamma)}_{\text{$d-\lambda$ times}}, \underbrace{(c_\beta, c_\beta)}_{\text{ $n-2-(2d - \lambda)$ times }}, \underbrace{(c_\gamma, c_\gamma)}_{\text{ $\lambda$ times}} \Rbag \Big)$
same for all \textit{edge type} 2-tuples.
\end{itemize}

From the analysis above, it is clear that all 2-tuples in the graph of the same initial type are assigned the same colour in the 1st iteration of 2-FWL. In other words, the vertices cannot be further partitioned, so the algorithm terminates. Therefore, if two SR graphs have the same parameters $n$,$d$,$\lambda$,$\mu$ then 2-FWL will yield the same colour distribution and thus the graphs will be deemed isomorphic.

\end{proof}

\section{Experimental Settings - Additional Details}\label{experiments_supp}

In this section, we provide additional implementation details of our experiments. All experiments were performed on a server equipped with 8 Tesla V100 16 GB GPUs, except for the Collab dataset where a Tesla V100 GPU with 32 GB RAM was used due to larger memory requirements.
Experimental tracking and hyperparameter optimisation were done via the Weights \& Biases platform (wandb) \cite{wandb}. Our implementation is based on native PyTorch sparse operations \cite{paszke2019pytorch} in order to ensure complete reproducibility of the results. PyTorch Geometric \cite{DBLP:journals/corr/abs-1903-02428} was used for additional operations (such as preprocessing and data loading).

In each one of the different experiments we aim to show that \textbf{structural identifiers can be used off-the-shelf and independently of the architecture.} At the same time we aim to suppress the effect of other confounding factors in the model performance, thus wherever possible we build our model on top of a baseline architecture. For more details, please see the relevant subsections. Interestingly, we observed that in most of the cases it was sufficient to replace only the first layer of the baseline architecture with a GSN layer, in order to obtain a boost in performance.

Throughout the experimental evaluation the structural identifiers $\mathbf{x}^V_v$ and $\mathbf{x}^E_{u,v}$ are one-hot encoded, by taking into account the unique count values present in the dataset. Other more sophisticated methods can be used, e.g. transformation to continuous features via a normalisation scheme or binning. However, we found that the number of unique values in our datasets were usually relatively small (which is a good indication of recurrent structural roles) and thus such methods were not necessary.

\subsection{Synthetic Experiment}

For the Strongly Regular graphs dataset (available from \url{http://users.cecs.anu.edu.au/~bdm/data/graphs.html}) we use all the available families of graphs with size of at most 35 vertices: 

\begin{itemize}
    \item SR(16,6,2,2): 2 graphs
    \item SR(25,12,5,6): 15 graphs
    \item SR(26,10,3,4): 10 graphs
    \item SR(28,12,6,4): 4 graphs
    \item SR(29,14,6,7): 41 graphs
    \item SR(35,16,6,8): 3854 graphs
    \item SR(35,18,9,9): 227 graphs
\end{itemize}

The total number of non-isomorphic pairs of the same size is $\approx 7*10^7$. We used a simple 2-layer architecture with width 64. The message aggregation was performed as in the general formulation of Eq. (5) of the main paper, where the update and the message functions are MLPs. The prediction is inferred by applying a sum readout function in the last layer and then passing the output through a MLP. Regarding the substructures, we use \textit{graphlet} counting, as certain \textit{motifs} (e.g. cycles of length up to 7) are known to be unable to distinguish strongly regular graphs (since they can be counted by the 2-FWL \cite{furer2017combinatorial, DBLP:conf/fct/ArvindFKV19}).

Given the adversities that strongly regular graphs pose in graph isomorphism testing, it would be interesting to see how this method can perform in other categories of hard instances, such as the classical  \textit{CFI} counter-examples for k-WL proposed in \cite{DBLP:journals/combinatorica/CaiFI92}, and explore further its expressive power and combinatorial properties. We leave this direction to future work.

\subsection{TUD Graph Classification Benchmarks}

For this family of experiments, due to the usually small size of the datasets, we choose a parameter-efficient architecture, in order to reduce the risk of overfitting. In particular, we follow the simple GIN architecture \cite{xu2018how} and we concatenate structural identifiers to vertex or edge features depending on the variant. Then for GSN-v, the hidden representation is updated as follows:
\begin{equation}\label{GSN-v-gin}
 {{\mathbf{h}^{t+1}_v =  \mathrm{UP}^{t+1}\Big([\mathbf{h}^t_v;\mathbf{x}^V_v] +   \sum_{u\in \mathcal{N}_v} [\mathbf{h}^t_u; \mathbf{x}^V_u]\Big)}}, 
\end{equation}
and for GSN-e:
\begin{equation}\label{GSN-e-gin}
    {\mathbf{h}^{t+1}_v =  \mathrm{UP}^{t+1}\Big([\mathbf{h}^t_v;\mathbf{x}^E_{v,v}] +   \sum_{u\in \mathcal{N}_v} [\mathbf{h}^t_u; \mathbf{x}^E_{u,v}]\Big)},
\end{equation}
where $\mathbf{x}^E_{v,v}$ is a dummy variable (also one-hot encoded) used to distinguish self-loops from edges. Empirically, we did not find training the $\epsilon$ parameter used in GIN to make a difference.

We implement an architecture similar to GIN \cite{xu2018how}, i.e. \textit{message passing layers}: 4 , \textit{jumping knowledge} from all the layers \cite{DBLP:conf/icml/XuLTSKJ18} (including the input), \textit{transformation of each intermediate graph-level representation}: linear layer, \textit{readout}: sum for biological and mean  for social networks. Vertex features are one-hot encodings of the categorical vertex labels. Similarly to the baseline, the hyperparameters search space is the following: \textit{batch size} in \{32, 128\} (except for Collab where only 32 was searched due to GPU memory limits), \textit{dropout} in \{0,0.5\}, \textit{network width} in \{16,32\} for biological networks, 64 for social networks, \textit{learning rate} in \{0.01, 0.001\}, \textit{decay rate} in \{0.5,0.9\} and \textit{decay steps} in \{10,50\} (number of epochs after which the learning rate is reduced by multiplying with the decay rate). For social networks, since they are not attributed graphs, we also experimented with using the degree as a vertex feature, but in most cases the structural identifiers were sufficient.

Model selection is done in two stages. First, we choose a substructure that we perceive as promising based on indications from the specific domain: \textit{triangles} for social networks and Proteins, and \textit{6-cycles (motifs)} for molecules. Under this setting we tune model hyperparameters for a GSN-e model. Then, we extend our search to the parameters related to the substructure collection: i.e. \textit{the maximum size $k$} and \textit{motifs vs graphlets}. In all the molecular datasets we search cycles with $k=3,\hdots, 12$, except for NCI1, where we also consider larger sizes due to the presence of large rings in the dataset (\textit{macrocycles} \cite{liu2017surveying}).  For social networks, we searched cliques with $k=3, 4, 5$.  
In Table \ref{tab:tud_datasets_hyperparams} we report the hyperparameters chosen by our model selection procedure, including the best performing substructures.

The seven datasets\footnote{more details on the description of the datasets and the corresponding tasks can be found at \cite{xu2018how}.} we chose are the intersection of the datasets used by the authors of our main baselines: the Graph Isomorphism Network (GIN) \cite{xu2018how}, a simple, yet powerful GNN with expressive power equal to the 1-WL test, and the Provably Powerful Graph Network (PPGN) \cite{maron2019provably}, a polynomial alternative to the Invariant Graph Network \cite{DBLP:conf/iclr/MaronBSL19}, that increases its expressive power to match the 2-FWL.
We also compare our results to other GNNs as well as Graph Kernel approaches. Our main baseline from the GK family is the Graph Neural Tangent Kernel (GNTK) \cite{DBLP:conf/nips/DuHSPWX19}, which is a kernel obtained from a GNN of infinite width. This operates in the Neural Tangent Kernel regime \cite{DBLP:conf/nips/JacotHG18, DBLP:conf/icml/Allen-ZhuLS19, DBLP:conf/icml/DuLL0Z19}.

\begin{table}[t]
    \centering
    \caption{Chosen hyperparameters for each of the two GSN variants for each dataset.}
    \label{tab:tud_datasets_hyperparams}
     \resizebox{\textwidth}{!}{
     {\small
    \begin{tabular}{p{0.8cm}p{3cm}|llll|lll}
        & Dataset & MUTAG & PTC & Proteins & NCI1 & Collab & IMDB-B & IMDB-M \\
        
      \hline
        
        \parbox[t]{1mm}{\multirow{9}{*}{GSN-e}} &  batch size & 32 & 128	& 32 & 32 & 32 & 32&  32\\
         
        & width  & 
         32 & 16 &  32 & 32 & 64 & 64 & 64\\
         
     & decay rate  & 0.9 & 0.5 & 0.5 & 0.9 & 0.5 &  0.5&  0.5 \\
     
        & decay steps & 50 & 50 & 10 & 10 & 50 &  10 & 10 \\
         
         & dropout & 0.5 & 0 & 0.5 & 0 & 0 & 0 & 0\\
         
        & lr & $10^{-3}$ & $10^{-3}$ & $10^{-2}$ &$10^{-3}$ &$10^{-2}$ & $10^{-3}$ &$10^{-3}$ \\
        
        & degree  & No & No & No & No & No & No & Yes\\

        & substructure type & 
          graphlets &  motifs & same & graphlets & same & same & same\\

        & substrucure family & cycles & cycles & cliques &cycles& clique & clique & cliques\\
        
        & k & 6	& 6 & 4 & 15 & 3 & 5 & 5\\

        \hline 
         \parbox[t]{1mm}{\multirow{9}{*}{GSN-v}} &  batch size & 32 & 128	& 32 & 32 & 32 & 32&  32\\
         
        & width  & 
         32 & 16 &  32 & 32 & 64 & 64 & 64\\
         
     & decay rate  & 0.9 & 0.5 &0.5 & 0.9 & 0.5 &  0.5&  0.5 \\
     
        & decay steps & 50 & 50 & 10 & 10 & 50 &  10 & 10 \\
         
         & dropout & 0.5 & 0 & 0.5 & 0 & 0 & 0 & 0\\
         
        & lr & $10^{-3}$ & $10^{-3}$ &$10^{-2}$  &$10^{-3}$ &$10^{-2}$ & $10^{-3}$ &$10^{-3}$ \\
        
        & degree  & No & No & No & No & No & Yes & Yes\\

        & substructure type & graphlets & graphlets & same& same & same & same & same\\

        & substrucure family & cycles & cycles & cliques &cycles& cliques & clique & cliques\\

        & k  &12 & 10& 4 & 3 & 3 & 4 & 3\\
        \bottomrule
    \end{tabular}}}
\end{table}

\subsection{Graph Regression on ZINC}\label{zinc_appendix}

\noindent\textbf{Experimental Details: }The ZINC dataset includes 12k molecular graphs of which 10k form the training set and the remaining 2k are equally split between validation and test (splits obtained from \url{https://github.com/graphdeeplearning/benchmarking-gnns}). Molecule sizes range from 9 to 37 vertices/atoms. Vertex features encode the type of atoms and edge features the chemical bonds between them. Again, here vertex and edge features are one-hot encoded. 

Our MPNN baseline model updates vertex representations as follows:  $\mathbf{h}^{t+1}(v) =  \text{MLP}^{t+1}(\mathbf{h}^{t}_v, \mathbf{m}^{t+1}_v)$,     $\mathbf{m}^{t+1}_v  = \sum_{u\in \mathcal{N}(v)} \text{MLP}^t\big(\mathbf{h}^{t}_v, \mathbf{h}^{t}_u, \mathbf{e}_{u,v}\big)$. Our instantiation of GSN is a simple extension where structural identifiers are also given as input to the message MLP. 

Following the same rationale as before, the network configuration is minimally modified w.r.t. the baselines provided in \cite{dwivedi2020benchmarking}, while here no hyperparameter tuning is done and we use the default ones provided by the authors. In particular, the parameters are the following: \textit{message passing layers}: 4, \textit{transformation of the output of the last layer}: MLP, \textit{readout}: sum, \textit{batch size}: 128, \textit{dropout}: 0.0, \textit{network width}: 128, \textit{learning rate}: 0.001. The learning rate is reduced by 0.5 (decay rate) after 5 epochs (\textit{decay rate
patience}) without improvement in the validation loss. Training is stopped when the learning rate reaches the \textit{minimum learning rate} value of $10^{-5}$.  Validation and test metrics are inferred using the model at the last training epoch. 

We select our best performing substructure related parameters based on the performance in the validation set in the last epoch. We search cycles with $k$ $=3,\hdots, 10$, \textit{graphlets vs motifs}, and \textit{GSN-v vs GSN-e}. The chosen hyperparameters for GSN are:  \textit{GSN-e, cycle graphlets of 10 vertices} and for GSN\textit{-EF}: \textit{GSN-v, cycle motifs of 8 vertices}. Once the model is chosen, we repeat the experiment 10 times with different seeds and report the mean and standard deviation of the test MAE in the last epoch.

\noindent\textbf{Disambiguation Scores:} In Table \ref{tab:uniqueness}, we provide the disambiguation scores $\delta$ as defined in section 5.5.1 of the main paper for different types of substructures. These are computed based on vertex structural identifiers (GSN-v).

\begin{table}[h]
         \centering
            \caption{Disambiguation scores $\delta$ on \textit{ZINC} for different substructure families and their maximum size $k$. Size $k=0$ refers to the use of the original vertex features only.}
          \begin{tabular}{r | l | l | l}
            k & Cycles & Paths & Trees \\
            \hline
            0 & 0.196 & 0.196 & 0.196\\
            3 & 0.199 & 0.540 & 0.540\\
            4 & 0.200 & 0.746 & 0.762\\
            5 & 0.256 & 0.866 & 0.875\\
            6 & 0.327 & 0.895 & 0.897\\
            7 & 0.330 & 0.900 & 0.900\\
            8 & 0.330 & 0.901 & 0.901\\
            9 & 0.330 & 0.901 & 0.901\\
            10 & 0.330 & 0.901 & 0.901\\
          \end{tabular}%
        \label{tab:uniqueness}
\end{table}

\subsection{Graph Classification on \texttt{ogbg-molhiv}}\label{app: molhiv}

The \texttt{ogbg-molhiv} dataset contains $\approx$ 41K graphs, with 25.5 vertices and 27.5 edges on average. As most molecular graphs, the average degree is small (2.2) and they exhibit a tree-like structure (average clustering coefficient 0.002). The average diameter is 12 (more details in \cite{DBLP:ogb}). Below we describe how we extend the two base architectures:

\medskip
\noindent\textbf{GIN+VN. } We follow the design choices of the authors of \cite{DBLP:ogb} and extend their architectures to include structural identifiers. Initial vertex and edge features are multi-hot encodings passed through linear layers that project them in the same embedding space, i.e. ${\mathbf{h}^{0}_v  = \mathbf{W}^0_{h}\cdot\mathbf{h}^{in}_v}$, ${\mathbf{e}^{t}_{v,u}  = \mathbf{W}^t_{e}\cdot\mathbf{e}^{in}_{u,v}}$.
The baseline model is a modification of GIN that allows for edge features: for each neighbour, the hidden representation is added to an embedding of its associated edge feature. Then the result is passed through a ReLU non-linearity which produces the neighbour's message. Formally, the aggregation is as follows:
\begin{equation}
    {\mathbf{h}^{t+1}_v =  \mathrm{UP}^{t+1}\left(\mathbf{h}^{t}_v + \sum_{u\in \mathcal{N}_v} \sigma\left(\mathbf{h}^t_u +  \mathbf{e}^{t}_{v,u}\right)\right)}
\end{equation}

In order to allow global information to be broadcasted to the vertices, a \textit{virtual node} takes part in the message passing. The virtual node representation, denoted as $\mathbf{G}^t$, is initialised as a zero vector $\mathbf{G}^0$ and then Message Passing becomes:
 \begin{equation}
    \begin{split}
       &\mathbf{\tilde{h}}^{t}_v = \mathbf{h}^{t}_v + \mathbf{G}^t, \
    \mathbf{h}^{t+1}_v =  \mathrm{UP}^{t+1}\left(\mathbf{\tilde{h}}^{t}_v + \sum_{u\in \mathcal{N}_v} \sigma\left(\mathbf{\tilde{h}}^t_u +  \mathbf{e}^{t}_{v,u}\right)\right), \ \\
 &\mathbf{G}^{t+1} = MLP^{t+1}\left(\mathbf{G}^t + \sum_{u\in \mathcal{N}_v} \mathbf{\tilde{h}}^t_u\right)
    \end{split}
\end{equation} 

We modify this model, as follows: first the substructure counts are embedded into the same embedding space as the rest of the features. Then, for GSN-v, they are added to the corresponding vertex embeddings:  ${\mathbf{\acute{h}}^{t}_v  = \mathbf{h}^{t}_v + \mathbf{W}^t_{V}\cdot\mathbf{x}^V_v}$, or for GSN-e, they are added to the edge embeddings $\mathbf{\acute{e}}^{t}_{v,u}  = \mathbf{e}^{t}_{v,u} + \mathbf{W}^t_{E}\cdot\mathbf{x}^E_{u,v}$. 

\medskip
\noindent\textbf{DGN + substructures. } 
We use the directional average operator as defined in \cite{beaini2020directional}:
\begin{equation}
    \mathbf{m}^{t+1}_v = [\sum_{u\in \mathcal{N}(v)} \alpha^1_{v,u}\mathbf{h}^{t}_u; \ldots ;\sum_{u\in \mathcal{N}(v)} \alpha^D_{v,u}\mathbf{h}^{t}_u],
\end{equation}
 where $\alpha^i_{v,u}$ are weighting average coefficients. In our case, each orbit $i$ induces a separate set of averaging coefficients. For example, for GSN-e 
 $\alpha_{v,u} = \frac{|x^{E}_{v,u}|}{\epsilon \ + \ \sum_{u\in \mathcal{N}(v)} |x^{E}_{v,u}|}$, where $x^{E}_{v,u}$ denotes edge-wise substructure counts (the index of the orbit $i$ was dropped to simplify notation). Similarly, for GSN-v, $\alpha_{v,u} = \frac{|x^{V}_{v} - x^V_{u}|}{\epsilon \ + \ \sum_{u\in \mathcal{N}(v)} |x^{V}_{v} - x^V_{u}|}$.
Subsequently, the vertex representation is updated as follows:  ${\mathbf{h}^{t+1}_v =  \text{MLP}^{t+1}(\mathbf{m}^{t+1}_v)}$. Observe that this model is simpler than the aforementioned,  in terms of both its parameter count and its expressive power. Since the MOLHIV dataset poses a significant challenge w.r.t. generalisation (the data splits reflect different molecular distributions), architectures biased towards simpler solutions usually perform better, sincey the mitigate the risk of overfitting.

In both cases we use the the same hyperparameters as the ones provided by the authors, %
and only select the substructure related parameters based on the highest validation ROC-AUC (choosing the best scoring epoch as in \cite{DBLP:ogb}). We search cycles with $k=3,\hdots, 12$, \textit{graphlets vs motifs}, and \textit{GSN-v vs GSN-e}. The chosen hyperparameters are: \textit{GSN-e, cycle graphlets of 6 vertices}.
We repeat the experiment 10 times with different seeds and report the mean and standard deviation of the train, validation and test ROC-AUC, again by choosing the best scoring epoch w.r.t the validation set.

\subsection{Structural Features \& Message Passing}\label{deepset_appendix}

The baseline architecture treats the input vertex and edge features, along with the structural identifiers, as a \textit{set}. In particular, we consider each graph as a set of independent edges $(v,u)$ endowed with the features of the endpoint vertices $\mathbf{h}_v, \mathbf{h}_u$, the structural identifiers $\mathbf{x}^V_v, \mathbf{x}^V_u$ and the edge features $\mathbf{e}(v,u)$, and we implement a DeepSets universal set function approximator \cite{zaheer2017deep} to learn a prediction function:
{\small
\begin{equation}
f\bigg(\Big\{\big(\mathbf{h}_v, \mathbf{h}_u, \mathbf{x}^V_v, \mathbf{x}^V_u, \mathbf{e}_{v,u}\big)\Big\}_{(v,u)\in\mathcal{E}_G}\bigg) = \rho\bigg(\sum_{(v,u)\in\mathcal{E}_G} \phi\big(\mathbf{h}_v, \mathbf{h}_u \mathbf{x}^V_v, \mathbf{x}^V_u, \mathbf{e}_{v,u}\big)\bigg),
\end{equation}
}
with $\mathcal{E}_G$ the edge set of the graph and $\rho, \phi$ MLPs. This baseline is naturally extended to the case where we consider edge structural identifiers by replacing $(\mathbf{x}^V_v, \mathbf{x}^V_u)$ with $\mathbf{x}^E_{v,u}$. For fairness of evaluation, we follow the exact same parameter tuning procedure as the one we followed for our GSN models for each benchmark, i.e. for the TUD datasets we first tune network and optimisation hyperaparameters (network width was set to be either equal to the ones we tuned for GSN, or such that the absolute number of learnable parameters was equal to those used by GSN; depth of the MLPs was set to 2) and subsequently we choose the substructure related parameters based on the evaluation protocol of \cite{xu2018how}. For ZINC and ogbg-molhiv we perform only substructure selection, based on the performance on the validation set. Using the same widths as in GSN leads to smaller baseline models w.r.t the absolute number of parameters, and we interestingly observed this to lead to particularly strong performance in some cases, especially Proteins and MUTAG, where our DeepSets implementation attains state-of-art results. This finding motivated us to explore `smaller' GSNs (with either reduced layer width or a single message passing layer). These GSN variants exhibited a similar trend, i.e. to perform better than their `larger' counterparts over these two datasets. We hypothesise this phenomenon to be mostly due to the small size of these datasets, which encourages overfitting when using architectures with larger capacity. In Table 4 in the main paper, we report the result for the best performing architectures, along with the number of learnable parameters.

\end{document}